%% file: manuscript.tex
\documentclass[letterpaper]{siamltex}
\usepackage[sectionbib]{natbib}
\usepackage{array,epsfig,fancyheadings,rotating,setspace}
\usepackage[footnotesize]{caption} 
\usepackage{epstopdf}
\usepackage[dvipsnames]{xcolor}
\usepackage{subfig}

\usepackage{amsmath}
\usepackage{amssymb}
\usepackage{amsfonts}
\usepackage{multirow}
\usepackage{algorithm}
\usepackage{booktabs}
\usepackage[normalem]{ulem}
\usepackage[flushleft]{threeparttable}
\usepackage{fullpage}
\usepackage[breaklinks=true]{hyperref}
\usepackage{breakcites}

\input defs.tex

\newcommand{\mb}[1]{\mathbf{#1}}
\def\fprod#1{\left\langle#1\right\rangle}

\def\B#1#2{\cB_{\norm{.}_#1}(#2)}
\def\Bc#1#2{\bar{\cB}_{\norm{.}_#1}(#2)}

\usepackage{float}
\usepackage[normalem]{ulem} 
\newcommand\str{\bgroup\markoverwith
{\textcolor{red}{\rule[0.5ex]{2pt}{1.5pt}}}\ULon} 
\newcommand\soutg{\bgroup\markoverwith
{\textcolor{green}{\rule[0.5ex]{2pt}{1.5pt}}}\ULon} 
\newcommand\stb{\bgroup\markoverwith
{\textcolor{blue}{\rule[0.5ex]{2pt}{1.5pt}}}\ULon} 

\DeclareMathOperator*{\argmin}{argmin}
\DeclareMathOperator*{\cov}{cov}

\begin{document}


\renewcommand{\baselinestretch}{2}


\markboth{\hfill{\footnotesize\rm Davanloo Tajbakhsh, S., Aybat, N. S., and del Castillo, E.} \hfill}
{\hfill {\footnotesize\rm Generalized SPS} \hfill}

$\ $\par


\fontsize{12}{14pt plus.8pt minus .6pt}\selectfont \vspace{0.8pc}
\centerline{\large\bf Generalized Sparse Precision Matrix Selection for Fitting Multivariate\footnote{This research is based on the dissertation of the first author, conducted under the guidance of the second and the third authors, Drs. Aybat and del Castillo.}}
\vspace{2pt} \centerline{\large\bf Gaussian Random Fields to Large Data Sets}
\vspace{.4cm} \centerline{S. Davanloo Tajbakhsh$^1$, N.S. Aybat$^2$, and E. del Castillo$^2$} \vspace{.4cm} \centerline{\it
$^{1}$The Ohio State University and $^2$The Pennsylvania State University} \vspace{.55cm} \fontsize{9}{11.5pt plus.8pt minus
.6pt}\selectfont

\vspace{-0.5cm}
\begin{quotation}
\noindent {\it Abstract:}
{We present a new method for estimating multivariate, second-order stationary Gaussian Random Field (GRF) models based on the Sparse Precision matrix Selection (SPS) algorithm, proposed by \cite{SPSarxiv} for estimating scalar GRF models.} Theoretical convergence rates for the estimated {between-response} covariance matrix and for the estimated parameters of the {underlying spatial} correlation function are established. {Numerical tests using simulated and real datasets} validate our theoretical findings. Data segmentation is used to handle large data sets.  \par

\vspace{9pt}
\par
\end{quotation}\par

\def\thefigure{\arabic{figure}}
\def\thetable{\arabic{table}}

\fontsize{12}{14pt plus.8pt minus .6pt}\selectfont

\section{Introduction}
\noindent Gaussian Random Field (GRF) models are very popular in Machine Learning, e.g.,~\citep{Rasmussen2006}, and are widely used in Geostatistics, e.g. \citep{Cressie2011}. They also have  applications in meteorology to model satellite data for forecasting or to solve inverse problems to tune weather models~\citep{Cressie2011}, or to model outputs of expensive-to-evaluate deterministic Finite Element Method (FEM) {computer codes}, e.g.,~\citep{DACE}. More recently, there have been applications of GRF to model stochastic simulations, e.g., queuing or inventory control models  \citep{Ankenman2010, Kleijnen2010}, or to model {free-form surfaces of manufactured}  products from noisy measurements for inspection or quality control purposes \citep{delCastillo2015}.

 In a GRF model, a key role is played by the covariance {or kernel} function which determines how the covariance between the process values at two locations   changes as the locations change across the process domain. There are many valid \emph{parametric} covariance functions, e.g., Exponential, Squared Exponential, {or Matern}; and Maximum Likelihood (ML) is the dominant method to estimate their parameters {from data (\cite{DACE})}. However, the ML fitting procedure suffers from two main challenges: \textbf{i}) the negative loglikelihood is a \emph{nonconvex} function of the covariance matrix; therefore, the covariance parameters may be poorly estimated, \textbf{ii}) the problem is computationally hard when the number of spatial locations $n$ is big. This is known as the \emph{``big-n"} problem in the literature. {Along with some other approximation methods, there is an important class that approximates the Gaussian likelihood using different forms of conditional independence assumptions which reduces the computational complexity significantly, e.g.,~\citep{Snelson05,Pourhabib14} and references therein.}

 {In ~\citep{SPSarxiv} we proposed the Sparse Precision Selection~(SPS) algorithm for \emph{univariate} processes to deal with the first challenge by providing theoretical guarantees on the SPS parameter estimates, and presented a \emph{segmentation} scheme on the training data to be able to solve big-$n$ problems. Given the nature of SPS, the segmentation does not result in discontinuities in the predicted process. In contrast, localized regression methods also rely on segmentation to reduce the computational cost; but, these methods may suffer from discontinuities on the predicted surface at the boundaries of the segments}. {In this paper, we present a Generalized SPS~(GSPS) method for fitting a {\em multivariate} GRF process that deals with the two aforementioned challenges when there are possibly cross-correlated multiple responses that occur at each spatial location.}


{Compared to SPS {(and also to GSPS)}, the likelihood approximation type GRF methods, e.g., \citep{Snelson05}, have the advantage of computational efficiency; {but, there are no guarantees on the quality of {the} parameter {estimates} as only an approximation to {the} likelihood function is optimized (compared to MLE, this is a small dimensional problem; but, still non-convex). On the other hand, {SPS has} theoretical error bound guarantees on hyper-parameter {estimates (this is also the case for GSPS, see Theorem~4 below)} -- note that these bounds also imply error guarantees on prediction quality through the mean of {the} predictive distribution.}}
\lhead[\footnotesize\thepage\fancyplain{}\leftmark]{}\rhead[]{\fancyplain{}\rightmark\footnotesize\thepage}

There is {a} wide variety of {applications that} require the approximation of a vector of \emph{correlated} responses obtained at each spatial or spatial-temporal location. {Climate models are classic Geostatistical examples where environmental variables such as atmospheric CO$_2$ concentration, ocean heat uptake and global surface temperature are jointly modeled (a simple such model is studied in \cite{UrbanKeller2010}).  Another classical application is environmental monitoring, for instance, \cite{Lin2008} uses a Multivariate GRF model to map spatial variations of five different heavy metals in soil. {This is} an application sharing a similar aim with 
Kriging in mining engineering where the spatial occurrence of two metals may be cross-correlated, e.g., silver and lead. Multivariate GRFs are also popular in multi-task learning \citep{Bonilla2007}, an area of machine learning  where multiple related tasks need to be learned so that simultaneously learning them can be better than learning them in isolation 
{without any} transfer of information between the tasks. The joint modeling of spatial responses is also useful in metrology when conducting {\em multi-fidelity analysis} \citep{Forrester2008}, where an expensive, high fidelity spatial response needs to be predicted from predominantly low fidelity responses, which are inexpensive -- see also \citep{Boyle2004}. Likewise, multivariate GRFs have been used to reconstruct 3-dimensional free-form surfaces of manufactured products {through} modeling each of the 3 coordinates of a measured point as a parametric surface response~\citep{delCastillo2015}. {Other applications of multivariate GRF include: \citep{wang2015} to model the response surface of a catalytic oxidation process with two highly correlated response variables; \citep{castellanos2015} to estimate low dimensional spatio-temporal patterns of finger motion in repeated reach-to-grasp movements; 
\citep{bhat2010} to study a multi-output GRF for computer model calibration with multivariate spatial data 
to infer parameters in a climate model.}
{Note that in many of such applications multiple realizations of the GRF {are} sensed/measured over time ($N>1$) over a fixed set of locations. {GRF applications with $N>1$ commonly arise in practice}, including those \textbf{i}) in ``metamodeling" of stochastic simulations for modeling an expensive-to-evaluate queuing or inventory control model, \textbf{ii}) in modeling product surfaces for inspection or quality control purposes, and \textbf{iii}) in models for which we observe a spatial process over time at the same locations 
{for a system known} to be static with respect to time.
}

{Rather} than considering each response independently, using the \emph{between-response} covariance can significantly enhance the prediction performance. {As mentioned by \cite{CressieBook}, the principle of exploiting co-variation to improve mean-squared prediction error goes back to Kolmogorov and Wiener in the first half of the XX century. It is well-known that the minimum-mean-square-error predictor of a single response component of a multivariate GRF involves the between-response covariances of all responses {\citep{DACE}}, a result {that lies} at the basis of the so-called Co-Kriging technique in Geostatistics~\citep{CressieBook}.}

{In this paper, we adopted a separable cross-covariance structure -- see~\eqref{eq:covY} -- which has been already adopted in the literature: \cite{Mardia93} proposed separability to model multivariate spatio-temporal data, and \cite{bhat2010} used separable cross-covariance for computer model calibration. This structure is also well known in the literature, see \citep{Gelfand2004,BanerjeeBook,gelfand2010} and \citep{genton2015}; moreover, \cite{Li08} even proposed a technique to test {the} {separability assumption} for a multivariate random process.}
{{Furthermore,} \cite{gelfand2010} mention one {additional use} of {a} separable covariance structure:
{\it ``A bivariate spatial process model using separability becomes appropriate for regression with a single covariate $X(s)$ and a univariate response $Y(s)$. In fact, we treat this as a bivariate process to allow for missing $X(s)$ for some observed $Y(s)$ and for inverse problems, inferring about $X(s_0)$ for a given $Y(s_0)$"}. {As an example of this type of application,} Banarjee and Gelfand have employed such {separable} models in \citep{banerjee2002,BanerjeeBook} to analyze {the} relationship between shrub density and dew duration for a dataset consisting of 1129 locations in a west-facing watershed in the Negev desert in Israel.}

However, fitting \emph{multivariate} GRFs not only suffers from the two challenges mentioned above; {in particular,} the parametrization of the matrix-valued covariance functions requires a higher-dimensional parameter vector which aggravates the difficulty of the GRF estimation problem further~{\citep{BanerjeeBook,Cressie2011}}. The goal of this paper is to extend the theory of the {univariate} SPS method~\citep{SPSarxiv} to include the {hyper-parameter} estimation of \emph{multivariate} GRF models {for which {the error bounds on the approximation quality} can be established}. The paper is organized as follows: Section 1.1 introduces the notation, and Section~2 provides some preliminary concepts related to the SPS method. In Section~3, {GSPS,} the multivariate generalization of the SPS method {is described and compared with other methods for fitting multivariate GRF}, and 
theoretical guarantees {of {the} GSPS estimates} are discussed. Section~4 includes numerical 
results. Finally, we summarize the main results in the paper and provide some future research directions in Section~5.

\subsection{Notation} Throughout the paper, given $x\in\reals^n$, $\norm{x}$, $\norm{x}_1$, $\norm{x}_\infty$ denote the Euclidean, $\ell_1$, and $\ell_\infty$ norms, respectively. For $x\in\reals^n$, $\diag(x)\in\mS^n$ denotes a diagonal matrix with its diagonal equal to $x$. Given $X\in\reals^{m\times n}$, we denote the vectorization of $X$ using $\vect(X)\in\reals^{np}$, obtained by stacking the columns of the matrix $X$ on top of one another. Moreover, let $r=\Rank(X)$, and $\sigma=[\sigma_i]_{i=1}^{r}\subset\reals^{r}_{++}$ (positive orthant) denote the singular values of $X$; then, $\norm{X}_F:=\norm{\sigma}$, $\norm{X}_2:=\norm{\sigma}_\infty$, and $\norm{X}_*:=\norm{\sigma}_1$ denote the Frobenius, spectral, and nuclear norms of $X$, respectively. Given $X,Y\in\reals^{m\times n}$, $\fprod{X,Y}:=\Tr(X^\top Y)$ denotes the standard inner product. Let $\cV$ be a normed vector space with norm $\norm{.}_a$. For $\bar{x}\in\cV$ and $r>0$, $\cB_{\norm{.}_a}(\bar{x},r):=\{x\in\cV:\ \norm{x-\bar{x}}_a<r\}$ denotes the open ball centered at $\bar{x}$ with radius $r>0$, and $\bar{\cB}_{\norm{.}_a}(\bar{x},r)$ denotes its closure.


\section{Preliminaries: the SPS method for a scalar GRF}
\noindent Let $\mathcal{X} \subseteq \mathbb{R}^d$ and $y:\cX\rightarrow\reals$ be a GRF, 
where $y(\mb{x})$ denotes the value of the process at location $\mb{x}\in\cX$. Let 
$m(\mb{x})=\mathbb{E}(y(\mb{x}))$ for $\mb{x}\in\cX$, and $c(\mb{x},\mb{x}')$ be {the spatial} covariance function denoting the covariance between $y(\mb{x})$ and $y(\mb{x}')$, i.e., $c(\mb{x},\mb{x}')=\mbox{cov}\left(y(\mb{x}),y(\mb{x}')\right)$ for all $\mb{x}, \mb{x}' \in\cX$. Without loss of generality, we assume that the GRF has a constant mean equal to zero, i.e., $m(\mathbf{x})=0$. Suppose the training data $\mathcal{D} = \{(\mb{x}_i,y^{(r)}_{i}): i=1,...,n,\ r=1,...,N\}$ contains $N$ realizations of the GRF at each of $n$ distinct locations in $\mathcal{D}^x:=\{\mb{x}_i\}_{i=1}^n\subset\cX$. Let $\mb{y}^{(r)}=[y_i^{(r)}]_{i=1}^n\in\reals^n$ denote the vector of $r$-th realization values for locations in $\cD^x$.

For simplicity in estimation, the covariance function, $c(\mb{x},\mb{x}')$, is typically assumed to belong to some parametric family $\{c(\mb{x},\mb{x}';\th,\nu):\ \th\in\Theta,\ \nu\geq 0\}$ and $c(\mathbf{x},\mathbf{x}',\th):=\nu \rho(\mb{x},\mb{x}',\boldsymbol{\theta})$, where $\rho(\mb{x},\mb{x}',\boldsymbol{\theta})$ is a parametric correlation function where $\th$  and $\nu$ denote the \emph{spatial correlation} and \emph{variance} parameters, respectively, and $\Theta\subset\reals^q$ is a set that contains the \emph{true} spatial correlation parameters -- see e.g. \cite{CressieBook}. Let $\th^*$ and $\nu^*$ denote the unknown \emph{true} parameters of the process. Given a set of locations $\cD^x=\{\mb{x}_i\}_{i=1}^n$, let $C(\th,\nu)\in\mS^n_{++}$ be such that its $(i,j)^{th}$ element is $c(\mb{x}_i,\mb{x}_j;\th,\nu)$ -- throughout, $\mS^n_{++}$ and $\mS^n_{+}$ denote the set of $n$-by-$n$ symmetric, positive definite and positive semidefinite matrices, respectively.

Let $C^*=C(\th^*,\nu^*)$ denote the true covariance matrix corresponding to locations in $\cD^x=\{\mb{x}_i\}_{i=1}^n$, and $P^*=(C^*)^{-1}$ denote the true \emph{precision matrix}.
In \cite{SPSarxiv}, we proposed a two-stage method, SPS, to estimate the unknown process parameters $\th^*$ and $\nu^*$.
\begin{figure}[t]
  \centering
  \vspace*{-0.3cm}
  \includegraphics[width=0.7\columnwidth]{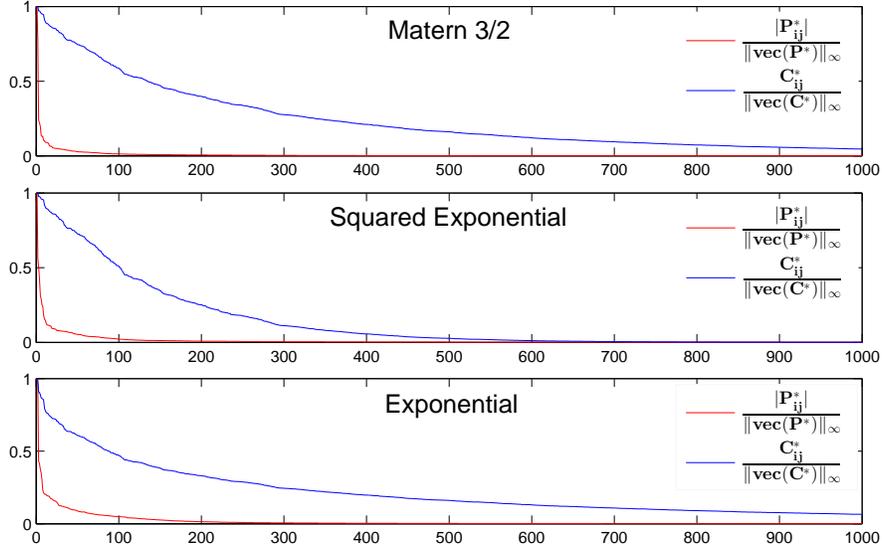}
  \vspace*{-0.1cm}
   \caption{{\scriptsize 
   Decaying behavior of elements of the Precision and Covariance matrices for GRFs. 
   The largest 1000 
   off-diagonal elements of the precision and covariance matrices (scaled by their maximums) plotted in descending order. 
   The underlying GRF was evaluated over 100 randomly selected points in $\cX=\{\mb{x}\in\reals^2: -50 \leq \mb{x}\leq 50\}$ for three covariance functions with range and variance parameters equal to 10, and 1, respectively.
   }} \label{fig:covSparsity}
   \vspace*{-0.25cm}
\end{figure}
{The method is 
motivated by the 
results in numerical linear algebra which demonstrate that if the elements of a matrix {show} a decay property, then {the elements of} its inverse also {show} a similar behavior -- see~\cite{benzi16,jaffard1990proprietes}. In particular, consider the two decay {classes} %
defined in~\cite{jaffard1990proprietes}: \vspace*{-0.1cm}
\begin{definition}
Given $\{\mb{x}_i\}_{i=1}^n\subset\cX$ and a metric $d:\cX\times\cX\rightarrow\reals_+$, a matrix $A\in\reals^{n\times n}$ belongs to the class $\cE_\gamma$ for some $\gamma>0$ if for all $\gamma'<\gamma$ there exists a constant $K_{\gamma'}$ such that 
$|A_{ij}|\leq K_{\gamma'}\ \exp\big(-\gamma' d(\mb{x}_i,\mb{x}_j)\big)$ for all $1\leq i,j\leq n$. 
Moreover, $A$ belongs to the class $\cQ_\gamma$ for some $\gamma>1$ if there exists a constant $K$ such that \vspace*{-5mm}
$|A_{ij}|\leq K\ \big(1+d(\mb{x}_i,\mb{x}_j)\big)^{-\gamma}$ for all $1\leq i,j\leq n$. \vspace*{-0.1cm}
\end{definition}
\begin{theorem}
\label{thm:decay}
Given $\{\mb{x}_i\}_{i=1}^n\subset\cX$ and a metric $d:\cX\times\cX\rightarrow\reals_+$, let $A\in\reals^{n\times n}$ be an invertible matrix. If $A\in\cE_\gamma$ for some $\gamma>0$, then $A^{-1}\in\cE_{\gamma'}$ for some $\gamma'>0$. Moreover, if $A\in\cQ_\gamma$ for some $\gamma>0$, then $A^{-1}\in\cQ_{\gamma}$. \vspace*{-0.1cm}
\end{theorem}
\begin{proof}
See Proposition~2 and Proposition~3 in~\cite{jaffard1990proprietes}. \vspace*{-0.4cm}
\end{proof}}
{This fast decay structure in the precision (inverse covariance) matrix of a GRF makes it a \emph{compressible signal}~\citep{candes2006compressive}; hence, one can argue that it can be well-approximated by a sparse matrix -- 
compare it with the covariance matrix {depicted} in Figure~\ref{fig:covSparsity}.}
{For} all 
stationary GRFs tested, we 
observed that for a finite set of {locations, the} magnitudes of the off-diagonal elements of the {\em precision} matrix decay 
to 0 
 much \emph{faster} than 
 {the elements of} the covariance matrix.

Let $a^*$ and $b^*$ be given constants such that $0\leq a^*\leq\sigma_{\min}(P^*)\leq\sigma_{\max}(P^*)\leq b^*\leq \infty$. In the first stage of the SPS algorithm, we proposed to solve the following convex loglikelihood problem penalized with a weighted $\ell_1$-norm  to estimate the true precision 
matrix corresponding to the given data locations $\cD^x$: \vspace*{-5mm}
\begin{equation}\label{eq:convexProgram}
\hat{P} := \argmin\{\fprod{S,P}-\log\det(P)+\alpha\fprod{G,|P|}: a^*\mb{I}\preceq P \preceq b^*\mb{I}\},
\vspace*{-3mm}
\end{equation}
where $S =\frac{1}{N}\sum_{r=1}^N \mb{y}^{(r)}{\mb{y}^{(r)}}^\top\in\mS^n_+$ is the sample covariance matrix. 
The weight matrix $G\in\mS^n$ is chosen as the matrix of pairwise distances: \vspace*{-5mm}
\begin{equation}
G_{ij} = \norm{\mb{x}_i-\mb{x}_j}, \quad \mbox{if $i\neq j$}, \quad
G_{ii} = \min\{\norm{\mb{x}_i-\mb{x}_j}:\ j\in\cI\setminus\{i\}\}, \label{eq:distMat}
\vspace*{-3mm}
\end{equation}
for all $(i,j)\in\cI\times \cI$, where $\cI=\{1,2,...,n\}$ and $|.|$ is the elementwise absolute value operator. The sparsity structure of the \emph{estimated} precision matrix $\hat{P}$ encodes the conditional independence structure of a Gaussian Markov Random {Field} (GMRF) approximation to the GRF. Using ADMM, the Alternating Direction Method of Multipliers, see~\citep{Boyd2011}, \eqref{eq:convexProgram} can be solved efficiently. {Indeed, since $-\log\det(.)$ is strongly convex and has a Lipschitz continuous gradient for $0< a^*\leq b^*<\infty$, ADMM iterate sequence converges \emph{linearly} to the optimal solution with a \emph{linear} rate~\citep{DengYin15}.}

In the second stage of the SPS method, we proposed to solve a least-square problem \eqref{eq:nonconvexFrobStationary} to estimate the unknown parameters $\boldsymbol{\theta}^*$ and $\nu^*$: \vspace*{-5mm}
\begin{equation} \label{eq:nonconvexFrobStationary}
(\hat{\th},\hat{\nu}) = \argmin_{\th\in\Theta,~\nu\geq 0} \norm{C(\th,\nu)-{\hat{P}}^{-1}}_F^2.
\vspace*{-3mm}
\end{equation}
In \cite{SPSarxiv}, we showed how to solve each optimization problem, and also established theoretical convergence rate of the SPS estimator.

SPS is therefore based on a Gaussian Markov Random Field (GMRF) approximation to the GRF. {While a GMRF on a \emph{lattice} can represent exactly a GRF under the conditional independence assumption,
this representation of a GRF can only be an approximation in a general continuous location space.}
{The} index set is countable for the lattice data, {but} the index set $\mathcal{X}$ for a GRF is \emph{uncountable}; hence, in general GMRF models cannot represent GRFs {\em exactly}.  \cite{LindgrenRueLindstrom2011} recently established
that the Matern GRFs are Markovian; in particular, they are Markovian when the smoothing parameter $\nu$ is such that $\nu-d/2\in\integers_+$, where $d$ is the dimension of the input space --
see \cite{LindgrenRueLindstrom2011} and
\cite{Fulgstad2015} for using this idea in the approximation of anisotropic and non-stationary GRFs.
Rather than using a triangulation of the input space as proposed by \cite{LindgrenRueLindstrom2011}, or assuming a lattice process, the \emph{first stage} of SPS lets the data determine the near-conditional independence pattern between variables through the precision matrix estimated via a weighted $\ell_1$-regularization.
Furthermore, this first stage helps to ``zoom into" the area where the true covariance parameters are located; hence, it helps not to get trapped in local optimum solutions in the {\emph{second stage}} of the method.

\section{Multivariate GRF Models}
\noindent From now on, let $y(\mathbf{x})\in \mR^p$ be the response vector at $\mathbf{x}\in\cX\subseteq\mR^d$ of a \emph{multivariate} Gaussian Random Field (GRF) $y:\cX\rightarrow\mR^p$ with \emph{zero} mean and a \emph{cross-covariance} function $c(\mathbf{x},\mathbf{x}')=\cov\left(y(\mathbf{x}),y(\mathbf{x}')\right)\in\mS^p_{++}$. The cross-covariance function is a crucial object in multivariate GRF models which should converge to a symmetric and positive-definite matrix as $\norm{\mathbf{x}-\mathbf{x}'}\rightarrow0$. {Similar} to the univariate case, the process is \emph{second-order stationarity} if $c(.,.)$ depends on $\mathbf{x}$ and $\mathbf{x}'$ only through $\mathbf{x}-\mathbf{x}'$, and it is \emph{isotropic} if $c(.,.)$ depends on $\mathbf{x}$ and $\mathbf{x}'$ only through $\norm{\mathbf{x}-\mathbf{x}'}$.

The parametric structure of the cross-covariance matrix should be such that the resulting cross-covariance matrix is a positive-definite matrix. \cite{Gelfand2004} and \cite{BanerjeeBook} review some methods to construct a valid cross-covariance function. 
In these methods, parameter estimation involves solving nonconvex optimization problems.

In this study, we assume a \emph{separable} cross-covariance function belonging to a parametric family, and propose a two-stage procedure for estimating the unknown parameters. The separable model assumes that the cross-covariance function is a multiplication of a spatial correlation function and a positive-definite between-response covariance matrix {(see~\citet{gelfand2010,Gelfand2004} and the references therein)}: \vspace*{-5mm}
\begin{equation} \label{eq:sepCovFunc}
c(\mathbf{x},\mathbf{x}')=\rho(\mathbf{x},\mathbf{x}')~\Gamma^*\in \mS^p_{+},
\vspace*{-5mm}
\end{equation}
where $\rho:\cX\times\cX\rightarrow[0,1]$ is the spatial correlation function, and $\Gamma^*\in\mS^p_{++}$ is the between-response covariance matrix. Furthermore, let $\by=[y(\mathbf{x}_1)^\top,...,y(\mathbf{x}_n)^\top]^\top\in\mR^{np}$ denote the process values in long vector form corresponding to locations in $\mathcal{D}^x:=\{\mb{x}_i\}_{i=1}^n\subset\cX$. Given the cross-covariance function \eqref{eq:sepCovFunc}, and the set of locations $\mathcal{D}^x$, $\by$ follows a multivariate Gaussian distribution with zero mean and covariance matrix equal to \vspace*{-5mm}
\begin{equation} \label{eq:covY}
C^*=R^*\otimes\Gamma^*,
\vspace*{-3mm}
\end{equation}
where $R^*\in\mS^n_{++}$ is the spatial correlation matrix such that $R^*_{ij}=\rho(\mathbf{x}_i,\mathbf{x}_j)$ for $i,j\in\cI:=\{1,\ldots,n\}$, and $\otimes$ denotes the Kronecker product. Hence, \vspace*{-5mm} 
\begin{equation} \label{eq:modelY}
\by\sim\cN(\bo,C^*).
\vspace*{-5mm}
\end{equation}

Let $\cD = \{(\mathbf{x}_i, y_i^{(r)}):\ i\in\cI,\ r=1,...,N\}$ be the training data set that contains $N$ realizations of the process over $n$ distinct locations $\mathcal{D}^x\subset\cX$, i.e., for each $r\in\{1,\ldots,N\}$, $\by^{(r)}=[y_i^{(r)}]_{i\in\cI}\in\reals^{np}$ is an independent realization of $\by=[y(\mb{x}_i)]_{i\in\cI}$. Hence, $\{\by^{(r)}\}_{r=1}^N$ are i.i.d. according to \eqref{eq:modelY}.

As in the univariate case, suppose the correlation function belongs to a parametric family $\{\rho(\mb{x},\mb{x}';\boldsymbol{\theta}):\ \boldsymbol{\theta}\in\Theta\}$, where $\Theta$ is a \emph{closed} \emph{convex} set containing the true parameter vector, $\th^*$, of the correlation function $\rho$. Given $\cD^x=\{\bx_i\}_{i\in\cI}$, define $R^*:=R(\th^*)$, where $R(\th)\in\mS^{n}_{++}$ is such that \vspace*{-5mm}
\begin{equation}
\label{eq:cor_matrix}
R(\th)=[r_{ij}(\th)]_{i,j\in\cI},\quad r_{ij}(\th)=\rho(\mb{x}_i,\mb{x}_j;\th)\quad \forall i,j\in\cI.
\vspace*{-3mm}
\end{equation}

{Consider a GRF model with all its parameters known, the {\emph{best linear unbiased} prediction at a new location $\bx_0$ is given by} the mean of the conditional distribution \newline $p(\by(\bx_0)|\{\by^{(r)}\}_{r=1}^N,\cD^x)$ which is \vspace*{-5mm}
\begin{equation}\label{eq:mvPred}
{\hat{\by}(\bx_0)=\large(\br(\bx_0;\th^*)^\top\otimes\Gamma^*\large)\large(R(\th^*)\otimes\Gamma^*\large)^{-1}\sum_{r=1}^N\by^{(r)}/N,} \vspace*{-5mm}
\end{equation}
where $\br(\bx_0;\th^*)\in\mR^n$ contains the spatial correlation between the new point $\bx_0$ and $n$ observed data points -- see~\citep{DACE}. It is important to note that the prediction equation is a \emph{continuous} function of the parameters $\th^*$ and $\Gamma^*$; hence, \emph{biased} estimation of the parameters will translate to poor prediction performance. Finally, the prediction formula~\eqref{eq:mvPred} shows the importance of considering {the} between-response covariance matrix {$\Gamma^*$} rather than using {$p$} independent univariate GRFs for prediction. Indeed, predicting each response independently of the others will result in suboptimal predictions.
}

The sample covariance matrix $S\in\mS^{np}_+$ is calculated as $S = \frac{1}{N}\sum_{r=1}^N\by^{(r)}{\by^{(r)}}^\top$. Furthermore, let $G\in\mS^n$ {be} such that $G_{ij}>0$ for all $i,j\in\cI$; in particular, we fix $G$ as in \eqref{eq:distMat} based on inter-distances. Let $P^*=(C^*)^{-1}$ be the true precision matrix corresponding to locations in $\cD^x$, and let $a^*$ and $b^*$ be some given constants such that $0\leq a^*\leq\sigma_{\min}(P^*)\leq\sigma_{\max}(P^*)\leq b^*\leq \infty$. To estimate 
$P^*$, we propose to solve the following convex program: \vspace*{-5mm}
\begin{equation}\label{eq:multVarCovexProgram}
\hat{P} = \argmin_{a^* I \preceq P\preceq b^*I} \fprod{S,P}-\log\det(P)+\alpha\fprod{G\otimes(\ones_p\ones_p^\top),|P|},
\vspace*{-3mm}
\end{equation}
where $|.|$ is the element-wise absolute value operator, and $\ones_p\in\reals^p$ denotes the vector of all ones. This objective penalizes the elements of the precision matrix with weights proportional to the distance between their locations. Problem \eqref{eq:multVarCovexProgram} can be solved efficiently using the ADMM implementation proposed in \cite{SPSarxiv}. Indeed, for $0< a^*\leq b^*<\infty$, the function $-\log\det(.)$ is strongly convex and has a Lipschitz continuous gradient; therefore, the ADMM sequence converges \emph{linearly} to the optimal solution -- see~\cite{DengYin15}.

Let $\hat{C}:=\hat{P}^{-1}$, and for all $(i,j)\in\cI\times\cI$ define block matrices $S^{ij}\in\mS^p$, $\hat{C}^{ij}\in\mS^p$ and $\Sigma^{ij}\in\mS^p$ such that $S=[S^{ij}]$, $\hat{C}=[\hat{C}^{ij}]$ and $C^*=[\Sigma^{ij}]$, i.e., $S^{ij}\in\mS^p$, $\hat{C}^{ij}\in\mS^p$ and $\Sigma^{ij}\in\mS^p$ are the sample, estimated and true covariance matrices between the locations $\mb{x}_i$ and $\mb{x}_j$. The following establishes a probability bound for the estimation error $\hat{P}-P^*$. 
\begin{theorem}\label{thm:statAnalysis1}
Let $\{\by^{(r)}\}_{r=1}^N\subset\mR^{nq}$ be independent realizations of a GRF with zero-mean and stationary covariance function $c(\mb{x},\mb{x}';\boldsymbol{\theta}^*)$ observed over $n$ distinct locations $\{\mb{x}_i\}_{i\in\cI}$ with $\cI:=\{1,...,n\}$; furthermore, let $C^*=R(\th^*)\otimes \Gamma^*$ be the true covariance matrix, and $P^*:={C^*}^{-1}$ be the corresponding true precision matrix, where $R(\th)$ is defined in \eqref{eq:cor_matrix}. Finally, let $\hat{P}$ be the {GSPS} estimator computed as in \eqref{eq:multVarCovexProgram} for some $G\in\mS^n$ such that $G_{ij}\geq 0$ for all $(i,j)\in\cI\times\cI$. 
Then for any given $M>0$, $N\geq N_0:=\left\lceil2\big[(M+2)\ln(np)+\ln 4\big]\right\rceil$, and $b^*\geq\sigma_{\max}(P^*)$, \vspace*{-5mm}
\begin{equation}\label{eq:probBoundGaussProcess}
\mbox{Pr}\left(\norm{\hat{P}-P^*}_F\leq 2 {b^*}^2p(n+\norm{G}_F)\alpha\right)\geq 1-(np)^{-M}, \vspace*{-3mm}
\end{equation}
for all $\alpha$ such that $40\max\limits_{i=1,...,p}(\Gamma_{ii}^*)\sqrt{\frac{N_0}{N}}\leq\alpha\leq 40\max\limits_{i=1,...,p}(\Gamma_{ii}^*)$. \vspace*{-5mm}
\end{theorem}
\begin{proof}
See the appendix.
\end{proof}

Given that $C^*=R^*\otimes\Gamma^*$, and the diagonal elements of the spatial correlation matrix $R^*$ are equal to one, we have $\Sigma^{ii}=\Gamma^*$. Therefore, we propose to estimate the between-response covariance matrix $\Gamma^*$ by taking the average of the $p\times p$ matrices along the diagonal of $\hat{C}$, i.e.,  \vspace*{-6mm}
\begin{equation}
\label{eq:gamma_hat}
\hat{\Gamma}:=\frac{1}{n}\sum_{i=1}^n\hat{C}^{ii}\in\mS^n_{++}.
\vspace*{-4mm}
\end{equation}
Note that \eqref{eq:multVarCovexProgram} implies that $\hat{P}\in\mS^{np}_{++}$; hence, $\hat{C}\in\mS^{np}_{++}$ as well. Therefore, all its block-diagonal elements are positive definite, i.e., $\hat{\Sigma}^{ii}\in\mS^n_{++}$ for $i=1,...,n$. Since $\hat{\Gamma}$ is a convex combination of $\hat{\Sigma}^{ii}\in\mS^n_+,\ i=1,...,n$ and the cone of positive definite matrices is a convex set, we also have $\hat{\Gamma}\in\mS^n_{++}$. A probability bound in the estimation error of the covariance matrices is shown in the following theorem. \vspace*{-2mm}
\begin{theorem}\label{thm:gammaConvergence}
Given $M>0$, $N\geq N_0:=\left\lceil2\big[(M+2)\ln(np)+\ln 4\big]\right\rceil$, and $a^*$, $b^*$ such that $0<a^*\leq \sigma_{\min}(P^*)\leq\sigma_{\max}(P^*)\leq b^*<\infty$, let $\hat{P}$ be the SPS estimator 
as in \eqref{eq:multVarCovexProgram}. Then $\hat{\Gamma}$, defined in \eqref{eq:gamma_hat}, and $\hat{C}=\hat{P}^{-1}$ satisfy
\vspace*{-5mm}
\begin{equation*}
\mbox{Pr}\left(\max\{\norm{\hat{C}-C^*}_2,~\norm{\hat{\Gamma}-\Gamma^*}_2\}\leq 2 \left(\frac{b^*}{a^*}\right)^2p(n+\norm{G}_F)\alpha\right)\geq 1-(np)^{-M}, \vspace*{-3mm}
\end{equation*}
for all $\alpha$ such that $40\max\limits_{i=1,...,p}(\Gamma_{ii}^*)\sqrt{\frac{N_0}{N}}\leq\alpha\leq 40\max\limits_{i=1,...,p}(\Gamma_{ii}^*)$.
\end{theorem}
\begin{proof}
From \eqref{eq:probBoundGaussProcess}, we have \vspace*{-5mm}
\begin{equation*}
\norm{\hat{C}-C^*}_2\leq \frac{1}{{a^*}^2}\norm{\hat{P}-P^*}_2 \leq \frac{1}{{a^*}^2}\norm{\hat{P}-P^*}_F\leq 2 \left(\frac{b^*}{a^*}\right)^2p(n+\norm{G}_F)\alpha,
\vspace*{-3mm}
\end{equation*}
where the first inequality follows from the Lipschitz continuity of $P\mapsto P^{-1}$ on the domain $P\succeq a^*\mathbf{I}$ with respect to the spectral norm $\norm{.}_2$. Hence, 
given that $\Gamma^*=\Sigma^{ii}$ for all $i\in\cI$, we have 
$\norm{\hat{C}^{ii}-\Gamma^*}_2\leq 2 \left(\frac{b^*}{a^*}\right)^2p(n+\norm{G}_F)\alpha$ 
for all $i\in\cI$.
Therefore, from convexity of $X\mapsto\norm{X-\Gamma^*}_2$, it follows that \vspace*{-5mm}
\begin{equation*}
\norm{\hat{\Gamma}-\Gamma^*}_2\leq \sum_{i\in\cI}\frac{1}{n}\norm{\hat{C}^{ii}-\Gamma^*}_2\leq 2 \left(\frac{b^*}{a^*}\right)^2p(n+\norm{G}_F)\alpha.
\vspace*{-1.2cm}
\end{equation*}
\end{proof}

\vspace*{-0.5cm}
\noindent \noindent {\bf Remark.}
For Theorems 2 and 3 to hold, $\alpha$ should belong to the interval $40\max\limits_{i=1,...,p}(\Gamma_{ii}^*)\sqrt{\frac{N_0}{N}}\leq\alpha\leq 40\max\limits_{i=1,...,p}(\Gamma_{ii}^*)$; for $N\geq N_0$ this interval is non-empty. The trade-off here is such that smaller $\alpha$ makes the estimation error bounds inside the probabilities tighter -- hence, desirable; however, at the same time, smaller $\alpha$ makes the estimated precision matrix less sparse which would require more memory to store a denser estimated precision matrix. Although the upper-bound on $\alpha$ is fixed, one can play with the lower bound; in particular, one can make it smaller by requiring more realizations $N$.

Given $\cD^x=\{\mb{x}_i\}_{i\in\cI}\subset\cX$, define $R:\reals^q\rightarrow\mS^n$ over $\Theta\subset\reals^q$
as in \eqref{eq:cor_matrix}, i.e., $R(\th)=[r_{ij}(\th)]_{i,j\in\cI}\in\mS^n$ and $r_{ij}(\th)=\rho(\mb{x}_i,\mb{x}_j;\th)$ for all $(i,j)\in\cI\times\cI$. To estimate the true parameter vector of the spatial correlation function, $\th^*$, we propose to solve
\vspace*{-8mm}
\begin{equation}
\label{eq:corrParamEst}
\hat{\th}\in\argmin_{\th\in\Theta} \frac{1}{2}\sum_{i,j\in\cI}\norm{r_{ij}(\th)\hat{\Gamma}-\hat{C}^{ij}}_F^2.
\vspace*{-5mm}
\end{equation}
The objective function of \eqref{eq:corrParamEst} can be written in a more compact form as the parametric function below, 
with parameters $\Gamma\in\mS^p$ and $C\in\mS^{np}$:
\vspace*{-5mm}
\begin{equation}
\label{eq:obj_stage2}
f(\th;\Gamma,C):=\frac{1}{2}\norm{R(\th)\otimes\Gamma-C}_F^2.
\vspace*{-3mm}
\end{equation}
Let $\th=[\theta_1,\ldots,\theta_q]^\top$, and $R_k':\reals^q\rightarrow\mS^n$ such that $R_k'(\th)=[\frac{\partial}{\partial\theta_k}r_{ij}(\th)]_{i,j\in\cI}$ for $k=1,\ldots,q$. Similarly, $R_{k\ell}^{''}:\reals^q\rightarrow\mS^n$ such that $R_{k\ell}^{''}(\th)=[\frac{\partial^2}{\partial\theta_k\partial\theta_\ell}r_{ij}(\th)]_{i,j\in\cI}$ for $1\leq k ,\ell\leq q$.
Let $Z(\th;\Gamma,C):=R(\th)\otimes\Gamma-C$; hence, $f(\th;\Gamma,C)=\norm{Z(\th;\Gamma,C)}_F^2/2$; and define $Z_k'(\th;\Gamma):=R_k'(\th)\otimes\Gamma$ for $k=1,\ldots,q$.
\begin{lemma}
\label{lem:strong_convexity}
Suppose $\rho(\mb{x},\mb{x}';\th)$ is twice continuously differentiable in $\th$ over $\Theta$ for all $\mb{x},\mb{x}'\in\cX$, then there exists $\gamma^*>0$ such that 
$\grad^2_\th f(\th^*;\Gamma^*,C^*)\succeq \gamma^*\mathbf{I}$ if and only if 
$\{\vect(R'_k(\th^*))\}_{k=1}^q\subset\reals^{n^2}$ are linearly independent.
\end{lemma}
\begin{proof}
Clearly,
$\grad_\th f(\th;\Gamma,C)=
\left[ \fprod{Z_1'(\th;\Gamma),Z(\th;\Gamma,C)}, \ldots, \fprod{Z_q'(\th;\Gamma),Z(\th;\Gamma,C)}
\right]^\top$.
\newline Hence, it can be shown that for $1\leq k\leq q$
\vspace*{-5mm}
\begin{equation}
\label{eq:grad_f_explicit}
\frac{\partial}{\partial \theta_k}f(\th;\Gamma,C) 
=\norm{\Gamma}_F^2\fprod{R_k'(\th),R(\th)}-\fprod{C,R_k'(\th)\otimes\Gamma},
\vspace*{-3mm}
\end{equation}
and from the product rule for derivatives, it follows that for $1\leq k,\ell\leq q$ \vspace*{-8mm}
{
\begin{equation}
\label{eq:hessian}
\frac{\partial^2}{\partial \theta_k \partial \theta_\ell}f(\th;\Gamma,C)
=\norm{\Gamma}_F^2\fprod{R_k'(\th),R_\ell'(\th)}+\fprod{R_{k\ell}^{''}(\th)\otimes\Gamma, R(\th)\otimes\Gamma-C}. \vspace*{-8mm}
\end{equation}
}%
Thus, since $C^*=r(\th^*)\otimes\Gamma^*$, we have \vspace*{-5mm}
\begin{equation*}
\frac{\partial^2}{\partial \theta_k \partial \theta_\ell}f(\th;\Gamma^*,C^*)
=\norm{\Gamma^*}_F^2\fprod{R_k'(\th^*),R_\ell'(\th^*)}.
\vspace*{-3mm}
\end{equation*}
Therefore, 
$\grad^2_\th f(\th^*;\Gamma^*,C^*)=\norm{\Gamma^*}_F^2~J(\th^*)^\top J(\th^*)$,
where $J(\th)\in\reals^{n^2\times q}$ such that $J(\th):=[\vect(R_1'(\th))\ldots \vect(R_q'(\th))]$. Hence,  there exists $\gamma^*>0$ such that $\grad^2_\th f(\th^*;\Gamma^*,C^*)\succeq\gamma^* I$ when $\{\vect(R_k'(\th^*))\}_{k=1}^q\subset\reals^{n^2}$ are linearly independent.
\end{proof}

\noindent {\bf Remark.} We 
comment on the linear independence condition stated in Lemma~\ref{lem:strong_convexity}. 
{For illustration purposes, } consider
the \emph{anisotropic exponential correlation} function $\rho(\bx,\bx',\th)=\exp\big(-(\bx-\bx')^\top\diag(\th)(\bx-\bx')\big)$, where $q=d$, and $\Theta=\reals^d_{+}$. Let $\cX=[-\beta,\beta]^d$ for some $\beta>0$, and suppose $\{\bx_i\}_{i\in\cI}$ is a set of independent identically distributed \emph{uniform} random samples inside $\cX$. Then it can be easily shown that for the anisotropic exponential correlation function, the condition in Lemma~\ref{lem:strong_convexity} holds with probability 1, i.e., $\{\vect(R_k'(\th^*))\}_{k=1}^d$ are linearly independent w.p.~1.


The next result builds on Lemma~\ref{lem:strong_convexity}, and it shows the convergence of the {GSPS} estimator as the number of samples per location, $N$, increases. 
\begin{theorem}
\label{thm:main}
Suppose $\th^*\in\intr\Th$, and $\rho(\mb{x},\mb{x}';\th)$ is twice continuously differentiable in $\th$ over $\Theta$ for all $\mb{x},\mb{x}'\in\cX$. Suppose $\{\vect(R'_k(\th^*))\}_{k=1}^q\subset\reals^{n^2}$ are linearly independent. For any given $M>0$ and $N\geq N_0:=\left\lceil2(M+2)\ln(np)+\ln 16\right\rceil$,
let $\hat{\th}^{(N)}$ be the {GSPS} estimator of $\th^*$, i.e., $\hat{\th}=\argmin_{\th\in\Theta} f(\th;\hat{\Gamma},\hat{C})$, and $\hat{\Gamma}$ be computed as in \eqref{eq:gamma_hat}.  
Then for any sufficiently small $\epsilon>0$, there exists $N\geq N_0$ satisfying $N=\cO(N_0/\epsilon^2)$
such that setting $\alpha=40\max_{i=1,\ldots,p}(\Gamma^*_{ii})\sqrt{\frac{N_0}{N}}$ in \eqref{eq:multVarCovexProgram} implies $\norm{\hat{\th}^{(N)}-\th^*}\leq\epsilon$ and $\norm{\hat{\Gamma}-\Gamma^*}=\cO(\epsilon)$ with probability at least $1-(np)^{-M}$; moreover, the STAGE-II function $f(\cdot;\hat{\Gamma},\hat{C})$ is strongly convex around the estimator
$\hat{\th}$. \vspace*{-3mm}
\end{theorem}
\begin{proof}
See the appendix. \vspace*{-3mm}
\end{proof}

\vspace{0.2cm}
\noindent \noindent {\bf Remark.} In Theorem 4, $\alpha$ is explicitly set equal to the lower bound, i.e., \newline $\alpha=40\max\limits_{i=1,...,p}(\Gamma_{ii}^*)\sqrt{\frac{N_0}{N}}=40\max\limits_{i=1,...,p}(\Gamma_{ii}^*)\sqrt{\frac{\left\lceil2\big[(M+2)\ln(np)+\ln 4\big]\right\rceil}{N}}$. Note that $M$ controls the probability bound; hence, the only unknown is $\max\limits_{i=1,...,p}(\Gamma_{ii}^*)$ -- we implicitly assume that this quantity can be estimated empirically or we have a prior knowledge about it. Moreover, Theorem~4 also guides us how to select $\alpha$. Indeed, both $\norm{\hat{\th}^{(N)}-\th^*}\leq \epsilon$ and $\norm{\hat{\Gamma}-\Gamma^*}=\cO(\epsilon)$ whenever $N=\cO(N_0/\epsilon^2)$; therefore, this implies we should set $\alpha=\cO(\epsilon)$. In the simulations provided in Section~4, $\alpha$ is set equal to $c\sqrt{\log(np)/N}$ where $c$ is chosen $10^{-2}$   {after} some preliminary cross-validation studies.

A summary of the proposed algorithm for fitting multivariate GRFs models is provided in Algorithm~\ref{alg:multGRF}.

\vspace*{-4mm}
{\singlespacing
\begin{algorithm}[h]
\caption{{GSPS} algorithm to fit multivariate GRFs}
\label{alg:multGRF}
\textbf{input}: $\mathcal{D}=\{(\mb{x}_i,\by_i^{(r)})\}_{i=1}^n\subset\cX\times\mR^p,\ i\in\cI,\ r=1,...,N\}$ \\
/* Compute the sample covariance and distance matrices*/\\
\hspace*{0.2in} $\by^{(r)}\gets[\by(\mathbf{x}_1)^T,...,\by(\mathbf{x}_n)^T]^T\in\mR^{np},\ r=1,...,N$\\
\hspace*{0.2in} $S\gets \frac{1}{N}\sum_{r=1}^N\by^{(r)}{\by^{(r)}}^\top$ \\
\hspace*{0.2in} $G_{ij}\gets \norm{\mb{x}_i-\mb{x}_j}_2, \quad \mbox{if}\ i\neq j$,\ $G_{ii}\gets \min\{\norm{\mb{x}_i-\mb{x}_j}_2:\ j\in\cI\setminus\{i\}$ \\
/* Compute the precision matrix and its inverse */ \\
\hspace*{0.2in} $\hat{P}\gets \argmin \{\fprod{S,P}-\log\det(P)+\alpha\fprod{G\otimes(\ones_q\ones_q^T),|P|}:\ a^*\mb{I}\preceq P \preceq b^*\mb{I}\}$ \\
\hspace*{0.2in}$\hat{C}\gets\hat{P}^{-1}$ \\
/* Compute the between response covariance matrix */ \\
\hspace*{0.2in} $\hat{\Gamma} \gets \frac{1}{n}\sum_{i\in\cI}\hat{C}^{ii}$ \\
/* Compute the spatial correlation parameter vector*/ \\
\hspace*{0.2in} $\hat{\th}\gets\argmin_{\th\in\Theta} \frac{1}{2}\sum_{i,j\in\cI}\norm{\rho_{ij}(\th)~\hat{\Gamma}-\hat{C}^{ij}}_F^2$ \\
\textbf{return}: $\hat{\Gamma}$ and $\hat{\th}$
\end{algorithm}
}

\vspace*{-2mm}
\noindent {\bf 3.1. Connection to SPS.} The main difference {between the SPS method and GSPS} is how $\hat{\Gamma}$, the estimator for $\Gamma^*$, is computed (when $p=1$, $\Gamma^*\in\reals_{++}$ corresponds to the variance parameter $\nu^*>0$ in SPS), and this difference in the way $\Gamma^*$ is estimated has significant implications on: a) the numerical stability of solving STAGE-II problem, and b) the proof technique to show consistency of the hyperparameter estimate as the number of process realization, $N$, increases.

{In the derivation of SPS, we {considered} the estimate $\hat{\nu}(\th)$ as an optimal response to the spatial correlation parameter $\th$, and show that $\hat{\nu}(\th)$ can be written in a closed form. In the second stage problem of SPS, given in \eqref{eq:nonconvexFrobStationary}, we solve a least squares problem over $\th$, i.e., \vspace*{-7mm}
$$\hat{\th}=\argmin_{\th\in\reals_+^d}\tfrac{1}{2}\sum_{i,j}(\hat{\nu}(\th)\rho(\bx_i,\bx_j,\th)-\hat{C}_{ij})^2.
\vspace*{-7mm}$$
Once $\hat{\th}$ is {computed}, we {estimate $\nu^*$} {using the best response function}: $\hat{\nu}=\hat{\nu}(\hat{\th})$. {The problem we {observed} with this approach in Davanloo et al. (2015) when applied to hyper-parameter estimation of a \emph{multivariate} GRF is that the second stage problem becomes challenging due to its {strong} nonconvexity,} 
which is {significantly aggravated relative to the univariate case due to the} multiplicative structure of $\hat{\Gamma}(\th)\rho(\bx_i,\bx_j,\th)$ {(when there is a single response, $p=1$, {this was not a problem for SPS)}. However, when $p>1$, {this same structure causes} numerical problems in the STAGE-II problem as one would need to solve \vspace*{-7mm}
\begin{equation}\label{eq:alternative}
\min_{\th\in\reals_+^d}\tfrac{1}{2}\norm{R(\th)\otimes\hat{\Gamma}(\th)-\hat{C}}_F^2. \vspace*{-5mm}
\end{equation}
Compared to {the} above problem, the STAGE-II problem we proposed in \eqref{eq:corrParamEst} for GSPS, i.e., $\min_{\th\in\reals_+^d}\tfrac{1}{2}\norm{R(\th)\otimes\hat{\Gamma}-\hat{C}}_F^2$,
behaves much better (although it is also non-convex in general), where $\hat{\Gamma}=\tfrac{1}{n}\sum_{i=1}^n\hat{C}_{ii}$ -- note that Theorem~\ref{thm:main} shows that {the} STAGE-II objective of GSPS is strongly convex around a neighborhood of the estimator. In all our numerical tests, standard nonlinear optimization techniques were able to compute a point close to the global minimizer very efficiently; however, this was not the case for the {problem} in \eqref{eq:alternative} when $p>1$ -- the same nonlinear optimization solvers we used for GSPS get stuck at a local minimizer far away from the global minimum. This is why we {propose} GSPS using {\eqref{eq:corrParamEst} in this paper}. Moreover, this new step of estimating $\hat{\Gamma}=\tfrac{1}{n}\sum_{i=1}^n\hat{C}_{ii}$ also helps us to give a much simpler proof for Theorem~4.}

{We {now comment on using GSPS to fit a \emph{multivariate} GRF as opposed to using SPS to fit $p$ independent \emph{univariate} GRFs to $p$ responses. As mentioned earlier, the latter can only be suboptimal in the presence of cross-covariances between the responses. Furthermore,} fitting a multivariate anisotropic GRF requires estimating $p(p+1)/2$ parameters for the between-response covariance matrix $\Gamma^*\in\mathbb{S}^p_{++}$ and $d$ parameters for the anisotropic spatial correlation function $\th\in\Theta\subseteq\mR^d_{++}$. On the other hand, fitting $p$ independent univariate anisotropic GRF requires estimating $p(d+1)$ parameters, i.e., for each univariate GRF one needs to estimate $d$ spatial correlation parameters and $1$ variance parameter. Therefore, if $d > \frac{p}{2}$, then fitting $p$ univariate GRF requires estimating more hyperparameters. Indeed, for some machine learning problems we have $d\gg p$, e.g., the classification problem for text categorization~\citep{joachims1998text} with $p>1$ related classes, and for these type of problems $d$ could be $\approx 10000$ and estimating $pd$ hyper-parameters will lead to \emph{overfitting}; hence, its prediction performance on test data will be worse compared to the prediction performance for multivariate GRF using~\eqref{eq:mvPred} {with $\th^*$ and $\Gamma^*$ replaced by $\hat{\th}$ and $\hat{\Gamma}$ which are computed as in~\eqref{eq:corrParamEst} and~\eqref{eq:gamma_hat}, respectively -- see Theorem~\ref{thm:main} for bounds on hyper-parameter approximation quality}. In Sections 4.2 and 4.3, the numerical tests conducted on simulated and real-data also show that the proposed GSPS method performs significantly better than modeling each response independently.}

\noindent{{\bf 3.2. Computational Complexity.} The computational bottleneck of GSPS method is the singular value decompositions~(SVD) that arises when solving the  STAGE-I problem using {the} ADMM algorithm. The per-iteration complexity is $\cO((np)^3)$. However, we should note that {the} STAGE-I problem is strongly convex; and ADMM has a linear rate~\citep{DengYin15}. Therefore, an $\epsilon$-optimal solution can be computed within $\cO(\log(1/\epsilon))$ iterations of ADMM. Thus, the overall complexity of solving STAGE-I is $\cO((np)^3~\log(1/\epsilon))$. Note that likelihood approximation methods do not have such iteration complexity results {due to the non-convexity of the approximate likelihood problem being solved}, even though they have cheaper per-iteration-complexity.} {In case of an isotropic process, the STAGE-II problem {in~\eqref{eq:corrParamEst} is one dimensional and it {can simply be} solved by using bisection.} If the process {is} anisotropic, then \eqref{eq:corrParamEst} is non-convex in general. That said, this problem is low dimensional due to $d\ll n$; hence, standard nonlinear optimization techniques can compute a local minimizer very efficiently -- note that we also show that STAGE-II objective is strongly convex around a neighborhood of the estimator. In all our numerical tests, STAGE-II problem is solved in much shorter time compared to STAGE-I problem; hence, it does not affect the overall complexity significantly. {In our code, we use golden-section search for isotropic processes, and Knitro's nonconvex solver to solve \eqref{eq:corrParamEst} for general anisotropic processes.}}

To eliminate $\cO((np)^3)$ complexity due to an SVD computation per ADMM iteration and due to computing $\hat{C}$, we used a segmentation scheme. We partition the data to $K$ segments, each one composed of $\approx n/K$ points chosen uniformly at random among $n$ locations, and assuming conditional independence between blocks. {In \citep{SPSarxiv}, we discussed two blocking/segmentation schemes: Spatial Segmentation (SS) and Random Selection (RS). Solving the STAGE-I problem with blocking schemes assumes a conditional independence assumption between blocks. In SS scheme such conditional independence assumption is potentially violated for points along {the common boundary between two blocks}. The RS scheme, however, works numerically better {for ``big-n"} scenarios. We believe that with RS scheme the infill asymptotics make the blocks conditionally independent to a reasonable degree.} {Using such blocking schemes,} the bottleneck complexity reduces to $\cO((np/K)^3)$ by solving STAGE-I problem for each block; hence, solving STAGE-I and computing $\hat{C}$, which we assume to be block diagonal, requires a total complexity of $\cO(\log(1/\epsilon)~(np)^3/K^2)$ and this bottleneck complexity can be controlled by properly choosing $K$.
\noindent {\bf 4. Numerical results}\par
\noindent {{In this section, comprehensive} simulation {analyses} are {reported for the study of} the performance of the proposed method. $N$ realizations of a zero-mean $p$-variate GRF with anisotropic spatial correlation function are simulated in a square domain $\cX=[0,10]^d$ over $n$ distinct points. The separable covariance function is the product of an anisotropic exponential spatial correlation function $\rho(\bx,\bx',\th^*)=\exp\big(-(\bx-\bx')^\top\diag(\th^*)(\bx-\bx')\big)$ and a $p$-variate between-response covariance matrix $\Gamma^*\in\mS^p_{++}$.} {The correlation function parameter vector $\th^*_{\ell}$ is sampled uniformly from the surface of a hyper-sphere in $\mR^d$ in the positive orthant for each replication $\ell\in\{1,...,L\}$. The between-response covariance matrix is $\Gamma^*_{\ell}=A^\top A$ for $A\in\mR^{w\times p}$ such that $w>p$, where the elements of $A$ are sampled independently from $\cN(0,1)$ per replication.} To solve the STAGE-I problem, the sparsity parameter $\alpha$ in \eqref{eq:convexProgram} is set equal to $c\sqrt{\log(np)/N}$ for some constant $c$. After some preliminary cross-validation studies, we set $c$ equal to $10^{-2}$. In our code, we use golden-section search for isotropic processes which requires a univariate optimization in STAGE-II, and use Knitro's nonconvex solver to solve \eqref{eq:corrParamEst} for general anisotropic processes. 

\noindent {\bf 4.1. Parameter estimate consistency}\par
{{We first} compare the quality of GSPS parameter estimate with the Maximum Likelihood Estimate (MLE). For 10 different replicates, we simulated $N$ independent realizations of GRF described above under different scenarios, and the mean of $\{\norm{\hat{\th}_\ell-\th^*}\}_{\ell=1}^{10}$ and $\{\norm{\hat{\Gamma}_\ell-\Gamma^*}_F\}_{\ell=1}^{10}$ are reported.}

{To deal with the nonconcavity of the likelihood, the MLEs are calculated from 10 random initial solutions and the best final solutions are reported. To solve problem in \eqref{eq:multVarCovexProgram} for the scenarios with $np>2000$, we used the {\em Random Selection} (RS) blocking scheme as described in \cite{SPSarxiv}. Tables~\ref{tbl:GSPSvsMLEvsCMGP4p2} and \ref{tbl:GSPSvsMLEvsCMGP4p5} show the results for {$p$-variate GRF models with $p=2$ and $p=5$,} respectively.}
\begin{table}[h!]
    \begin{center}
    \scriptsize
    \begin{threeparttable}
    \begin{tabular}[t]{cccccccccc}
    \toprule
     &  &  &\multicolumn{2}{c}{N=1} & \multicolumn{2}{c}{N=10} & \multicolumn{2}{c}{N=40} & \\
     \cmidrule(r){4-5} \cmidrule(r){6-7} \cmidrule(r){8-9}
     d & n & Method & \tiny $\norm{\hat{\th}_l-\th^*}_2$ & \tiny $\norm{\hat{\Gamma}_\ell-\Gamma^*}_F$ & \tiny $\norm{\hat{\th}_l-\th^*}_2$ & \tiny $\norm{\hat{\Gamma}_\ell-\Gamma^*}_F$ &  \tiny $\norm{\hat{\th}_l-\th^*}_2$ & \tiny $\norm{\hat{\Gamma}_\ell-\Gamma^*}_F$ & Time (sec) \\
    \cmidrule(r){1-1} \cmidrule(r){2-2} \cmidrule(r){3-3} \cmidrule(r){4-5} \cmidrule(r){6-7} \cmidrule(r){8-9} \cmidrule(r){10-10}
    \multirow{6}{*}{2} & \multirow{2}{*}{100}&GSPS&0.43&0.89&0.34&0.66&0.21&0.53&14.9\\
      &  & MLE & 0.38&0.78&0.36&0.70&0.26&0.61& 21.3\\
      \cmidrule(r){2-3} \cmidrule(r){4-5} \cmidrule(r){6-7} \cmidrule(r){8-9} \cmidrule(r){10-10}
      & \multirow{2}{*}{500} & GSPS & 0.39&0.81&0.29&0.60&0.13& 0.43 & 312.3 \\
      &  & MLE & 0.37&0.83&0.32&0.62&0.19& 0.50 & 496.1\\
      \cmidrule(r){2-3} \cmidrule(r){4-5} \cmidrule(r){6-7} \cmidrule(r){8-9} \cmidrule(r){10-10}
      & \multirow{2}{*}{1000} & GSPS & 0.33&0.73&0.23&0.57&0.08& 0.34 & 2342.5 \\
      &  & MLE & 0.32&0.74&0.28&0.58&0.11 & 0.40 & 3216.5 \\
   \cmidrule(r){1-10}
    \multirow{6}{*}{5} & \multirow{2}{*}{100} & GSPS & 0.49&0.96&0.38&0.71&0.26& 0.56 & 18.9  \\
      &  & MLE & 0.46&0.93&0.42&0.71&0.36 & 0.61 & 36.5\\
     \cmidrule(r){2-3} \cmidrule(r){4-5} \cmidrule(r){6-7} \cmidrule(r){8-9} \cmidrule(r){10-10}
      & \multirow{2}{*}{500} & GSPS & 0.44&0.88&0.33&0.69&0.29 & 0.53 & 527.4 \\
      &  & MLE & 0.46&0.89&0.38&0.67&0.34& 0.59 & 1023.4\\
      \cmidrule(r){2-3} \cmidrule(r){4-5} \cmidrule(r){6-7} \cmidrule(r){8-9} \cmidrule(r){10-10}
      & \multirow{2}{*}{1000} & GSPS & 0.40&0.81&0.30&0.62&0.29 & 0.50 & 2987.3\\
      &  & MLE & 0.43&0.92&0.35&0.66&0.34 & 0.56 & 6120.8\\
   \cmidrule(r){1-10}
    \multirow{6}{*}{10} & \multirow{2}{*}{100} & GSPS & 0.55&1.05&0.39&0.82&0.35 & 0.58 & 29.1 \\
      &  & MLE & 0.57&1.02&0.56&0.89&0.53 & 0.69 & 75.2 \\
      \cmidrule(r){2-3} \cmidrule(r){4-5} \cmidrule(r){6-7} \cmidrule(r){8-9} \cmidrule(r){10-10}
      & \multirow{2}{*}{500} & GSPS & 0.47&0.99&0.35&0.73&0.31& 0.49 & 613.8 \\
      &  & MLE & 0.54&1.00&0.53&0.81&0.50 & 0.58 & 4125.6\\
      \cmidrule(r){2-3} \cmidrule(r){4-5} \cmidrule(r){6-7} \cmidrule(r){8-9} \cmidrule(r){10-10}
      & \multirow{2}{*}{1000} & GSPS & 0.41&0.89&0.31&0.71&0.29 & 0.43 & 4920.5 \\
      &  & MLE & 0.51&0.97&0.49&0.76&0.47 & 0.50 & 7543.3\\
    \bottomrule
    \end{tabular}
    \end{threeparttable}
    \end{center}
    \vspace*{-0.2cm}
     \caption{Comparison of GSPS vs. MLE for p=2 response variables}
    \label{tbl:GSPSvsMLEvsCMGP4p2}
\end{table}

\begin{table}[h!]
    \begin{center}
    \scriptsize
    \begin{threeparttable}
    \begin{tabular}[t]{cccccccccc}
    \toprule
     &  &  &\multicolumn{2}{c}{N=1} & \multicolumn{2}{c}{N=10} & \multicolumn{2}{c}{N=40} & \\
     \cmidrule(r){4-5} \cmidrule(r){6-7} \cmidrule(r){8-9}
     d & n & Method & \tiny $\norm{\hat{\th}_l-\th^*}_2$ & \tiny $\norm{\hat{\Gamma}_\ell-\Gamma^*}_F$ & \tiny $\norm{\hat{\th}_l-\th^*}_2$ & \tiny $\norm{\hat{\Gamma}_\ell-\Gamma^*}_F$ &  \tiny $\norm{\hat{\th}_l-\th^*}_2$ & \tiny $\norm{\hat{\Gamma}_\ell-\Gamma^*}_F$ & Time (sec) \\
    \cmidrule(r){1-1} \cmidrule(r){2-2} \cmidrule(r){3-3} \cmidrule(r){4-5} \cmidrule(r){6-7} \cmidrule(r){8-9} \cmidrule(r){10-10}
    \multirow{6}{*}{2} & \multirow{2}{*}{100}&GSPS&0.66&1.43&0.38&0.91&0.30&0.76&17.2\\
      &  & MLE & 0.62&1.40&0.57&1.30&0.41&1.28&26.3\\
      \cmidrule(r){2-3} \cmidrule(r){4-5} \cmidrule(r){6-7} \cmidrule(r){8-9} \cmidrule(r){10-10}
      & \multirow{2}{*}{500} & GSPS & 0.58&1.35&0.35&0.87&0.27& 0.73& 363.4 \\
      &  & MLE & 0.57&1.32&0.51&1.24&0.39&1.15& 512.5\\
      \cmidrule(r){2-3} \cmidrule(r){4-5} \cmidrule(r){6-7} \cmidrule(r){8-9} \cmidrule(r){10-10}
      & \multirow{2}{*}{1000} & GSPS & 0.49&1.24&0.31&0.82&0.24&0.70& 2835.4 \\
      &  & MLE & 0.49&1.22&0.42&1.19&0.33 & 1.10 & 3913.7 \\
   \cmidrule(r){1-10}
    \multirow{6}{*}{5} & \multirow{2}{*}{100} & GSPS & 0.73&1.49&0.50&0.92&0.39& 0.79& 25.6  \\
      &  & MLE & 0.71&1.47&0.62&1.36&0.49 & 1.35 & 53.1\\
     \cmidrule(r){2-3} \cmidrule(r){4-5} \cmidrule(r){6-7} \cmidrule(r){8-9} \cmidrule(r){10-10}
      & \multirow{2}{*}{500} & GSPS & 0.60&1.41&0.44&1.00&0.36& 0.75& 665.6 \\
      &  & MLE & 0.64&1.43&0.54&1.26&0.44& 1.24& 1424.3\\
      \cmidrule(r){2-3} \cmidrule(r){4-5} \cmidrule(r){6-7} \cmidrule(r){8-9} \cmidrule(r){10-10}
      & \multirow{2}{*}{1000} & GSPS & 0.54&1.32&0.39&1.06&0.31 & 0.74& 3783.6\\
      &  & MLE & 0.63&1.36&0.47&1.20&0.38 & 1.17& 7346.7\\
   \cmidrule(r){1-10}
    \multirow{6}{*}{10} & \multirow{2}{*}{100} & GSPS & 0.77&1.57&0.59&0.98&0.52 & 0.85 & 45.3 \\
      &  & MLE & 0.79&1.60&0.67&1.39&0.61 & 1.43& 87.2 \\
      \cmidrule(r){2-3} \cmidrule(r){4-5} \cmidrule(r){6-7} \cmidrule(r){8-9} \cmidrule(r){10-10}
      & \multirow{2}{*}{500} & GSPS & 0.65&1.47&0.54&1.03&0.46& 0.81 & 717.6 \\
      &  & MLE & 0.74&1.56&0.60&1.31&0.52 & 1.37& 4994.3\\
      \cmidrule(r){2-3} \cmidrule(r){4-5} \cmidrule(r){6-7} \cmidrule(r){8-9} \cmidrule(r){10-10}
      & \multirow{2}{*}{1000} & GSPS & 0.59&1.39&0.49&1.08&0.42 & 0.75 & 6001.3 \\
      &  & MLE & 0.66&1.48&0.53&1.27&0.45 & 1.29& 8223.1\\
    \bottomrule
    \end{tabular}
    \end{threeparttable}
    \end{center}
    \vspace*{-0.2cm}
     \caption{Comparison of GSPS vs. MLE for p=5 response variables}
    \label{tbl:GSPSvsMLEvsCMGP4p5}
\end{table}

{{For fixed $n$}, the parameter {estimation} {error increases} with the dimension of the input space $d$, which is reasonable due to higher number of parameters in the anisotropic correlation function. Furthermore, the errors {increase} with $p$, the number of responses. As expected, increasing the point density $n$ helps in {improving the} estimation of the parameters, i.e., reducing the errors, a result in accordance to the expected effect of infill asymptotics.}

{Overall, the GSPS method results in better parameter estimates compared to MLE with relative performance improvements becoming more obvious as $p$ and $d$ increase. Furthermore, as the number of realizations $N$ increases GSPS performs consistently better than MLE. Note that the robust performance of the proposed method is theoretically guaranteed for $N\geq N_0$ from Theorem~\ref{thm:main}.}

\noindent {\bf 4.2. Prediction consistency}\par
{To evaluate prediction performance, we compared {the} GSPS method {against using} multiple \emph{univariate} SPS (mSPS) {fits} and {against} the Convolved Multiple output Gaussian Process (CMGP) {method} by \cite{Alvarez2011}. Given the size of the training data $n$, {{none of the approximations in \citep{Alvarez2011} with induced points} 
were used, this corresponds to what Alvarez and Lawrence refer as the CMGP method.}}

{For 10 different replicates, we simulated $N$ independent realizations of the same GRF, {which is defined at the beginning of Section~4}, under different scenarios to learn the model parameters.
We also simulated the p-variate response over {a} fixed set of $n_0=1000$ test locations per replicate. The mean of the conditional distribution {$p(\by(\bx_0)|\{\by^{(r)}\}_{r=1}^N,\cD^x)$} is used to predict at these test locations and, then, the mean of Mean Squared Prediction Error (MSPE) over 10 replicates, $p$ outputs, and $n_0$ test points are reported for $p=2$ and $p=5$ in Tables \ref{tbl:MSE_GSPSvsCMGP_p2} and \ref{tbl:MSE_GSPSvsCMGP_p5}, respectively.} 

\begin{table*}[h!]
    \scriptsize
    \begin{threeparttable}
    \caption{MSPE comparison for $p=2$ response variables}
    \label{tbl:MSE_GSPSvsCMGP_p2}
    \begin{tabular}[t]{cccccc}
    \toprule
     d & n & Method & $N=1$ & $N=10$ & $N=40$ \\
    \cmidrule(r){1-1} \cmidrule(r){2-2} \cmidrule(r){3-3} \cmidrule(r){4-6}
    \multirow{6}{*}{2} & \multirow{3}{*}{100}&mSPS&7.02&2.68&2.08\\
      &  & GSPS & 6.71&2.12&1.44\\
      &  & CMGP & 6.40&2.39&1.61\\
      \cmidrule(r){2-2} \cmidrule(r){3-3}\cmidrule(r){4-6}
      & \multirow{3}{*}{400} & mSPS & 6.76&2.22&1.87 \\
      &  & GSPS & 5.53&1.89&0.91\\
      &  & CMGP & 5.16&2.04&1.33\\
   \cmidrule(r){1-6}
    \multirow{6}{*}{5} & \multirow{3}{*}{100} & mSPS & 7.12&3.09&2.39  \\
      &  & GSPS & 6.98&2.45&1.52 \\
      &  & CMGP & 6.74&2.95&1.99\\
     \cmidrule(r){2-2} \cmidrule(r){3-3} \cmidrule(r){4-6}
      & \multirow{3}{*}{400} & mSPS & 7.34&3.04&2.24 \\
      &  & GSPS & 5.88&2.45&1.05\\
      &  & CMGP & 6.32&2.89&1.73\\
   \cmidrule(r){1-6}
    \multirow{6}{*}{10} & \multirow{3}{*}{100} & mSPS & 7.83&4.15&3.23 \\
     &  & GSPS & 7.11&3.34&2.02 \\
     &  & CMGP & 6.97&3.67&2.39\\
      \cmidrule(r){2-2} \cmidrule(r){3-3} \cmidrule(r){4-6}
      & \multirow{3}{*}{400} & mSPS & 7.65&3.53&2.65 \\
      &  & GSPS & 6.13&2.96&1.22\\
      &  & CMGP & 6.63&3.32&2.28\\
    \bottomrule
    \end{tabular}
    \end{threeparttable}
    \qquad
    \begin{threeparttable}
    \caption{MSPE comparison for $p=5$ response variables}
    \label{tbl:MSE_GSPSvsCMGP_p5}
    \begin{tabular}[t]{cccccc}
    \toprule
     d & n & Method & $N=1$ & $N=10$ & $N=40$ \\
    \cmidrule(r){1-1} \cmidrule(r){2-2} \cmidrule(r){3-3} \cmidrule(r){4-6}
    \multirow{6}{*}{2} & \multirow{3}{*}{100}&mSPS&7.83&4.42&3.08\\
      &  & GSPS & 7.05&3.89&2.11\\
      &  & CMGP & 6.74&3.71&2.49\\
      \cmidrule(r){2-2} \cmidrule(r){3-3}\cmidrule(r){4-6}
      & \multirow{3}{*}{400} & mSPS & 7.51&3.78&2.18 \\
      &  & GSPS & 6.81&2.96&1.32\\
      &  & CMGP & 6.23&3.36&2.03\\
   \cmidrule(r){1-6}
    \multirow{6}{*}{5} & \multirow{3}{*}{100} & mSPS & 8.54&5.30&3.32  \\
      &  & GSPS & 7.19&4.43&2.01 \\
      &  & CMGP & 7.10&4.97&2.86\\
     \cmidrule(r){2-2} \cmidrule(r){3-3} \cmidrule(r){4-6}
      & \multirow{3}{*}{400} & mSPS & 8.22&4.15&2.63 \\
      &  & GSPS & 7.00&3.10&1.45\\
      &  & CMGP & 7.45&4.04&2.65\\
   \cmidrule(r){1-6}
    \multirow{6}{*}{10} & \multirow{3}{*}{100} & mSPS & 9.23&5.67&3.43 \\
      &  & GSPS & 7.23&4.68&2.19 \\
      &  & CMGP & 8.53&5.25&3.24\\
      \cmidrule(r){2-2} \cmidrule(r){3-3} \cmidrule(r){4-6}
      & \multirow{3}{*}{400} & mSPS & 8.54&4.24&2.94 \\
      &  & GSPS & 7.08&3.23&1.63\\
      &  & CMGP & 7.82&4.20&2.87\\
    \bottomrule
    \end{tabular}
    \end{threeparttable}
    \vspace*{-0.2cm}
    \caption*{The mean of the Mean Squared Prediction Error (MSPE) comparison of multiple SPS (mSPS),  Generalized SPS (GSPS) and Convolved Multiple Gaussian Process (CMGP) of Alvarez and Lawrence (2011) for $p$ response variables}
\end{table*}

%

One important observation is that the prediction performance of GSPS is almost ubiquitously better than mSPS method. This means that learning the {cross-covariance} between different responses provides {additional useful} information that helps improve the prediction performance of the joint model, GSPS, over mSPS. Comparing GSPS vs. CMGP, we observe relatively better performance of CMGP over GSPS {when $N=1$ in a lower dimensional input space, e.g., $(N,d)=(1,2)$. {However,} as $n$, the number of locations, increases, the GSPS predictions become better than CMGP even if $N=1$, e.g., for $(N,d)=(1,5)$, GSPS does better than CMPG for $n=400$.} The prediction performance of GSPS improves significantly with increasing $N$, the number of realizations of the process. In $d=10$ dimensional space, GSPS is performing consistently better, even when $N=1$ for both $p=2$ and $p=5$. However, we should note that CMGP with 50 inducing points is significantly faster than GSPS in the learning phase.

\noindent {\bf 4.3. Real data set}\par
{We now use a real data set to compare} the prediction performance of GSPS {with the naive method of using} multiple univariate SPS (mSPS) {fits}, and {with the} two {approximation} methods proposed in Alvarez and Lawrence (2011). The data set consists of n=9635 $(x, y, z)$ measurements {obtained by a laser scanner from a free-form} surface of a manufactured product. \cite{delCastillo2015} proposed modeling each coordinate, separately, as a function of the corresponding $(u, v)$ {surface coordinates (obtained using the ISOMAP algorithm by \cite{Tenenbaum2000})}. These $(u,v)$ {coordinates} are {selected} such that their pairwise Euclidean distance is equal to the {pairwise} geodesic distances between their corresponding $(x, y, z)$ points {along the surface.} We first model $(x(u,v), y(u,v), z(u,v))$ as a multivariate GRF using GSPS and compare {against} fitting $p=3$ independent {\em univariate} GRF using {the} SPS method (mSPS).

Given the large size of the data set, n=9635, we use {the} {\em Random Selection} blocking scheme as described in \cite{SPSarxiv} for varying number of blocks; hence, {there are} different number of observations per block. Table~\ref{tbl:MSE_realdata_1} reports the MSPE and the corresponding standard errors (std. error) obtained from 10-fold cross validation.
\begin{table}[h!]
    \begin{center}
    \begin{tabular}{lccc}
    \toprule
     Method & n/block & MSPE & std. error \\
     \hline
     mSPS & 100   & 0.0932 & 0.0047 \\
     mSPS & 500   & 0.0621 & 0.0021 \\
     mSPS & 1000 & 0.0842 & 0.0013 \\
     \hline
     GSPS & 100   & 0.0525 & 0.0023 \\
     GSPS & 500   & 0.0167 & 0.0012 \\
     GSPS & 1000 & 0.0285 & 0.0019 \\
    \bottomrule
    \end{tabular}
    \end{center}
    \vspace{-0.25cm}
     \caption{10-fold cross validation to evaluate prediction performance of multiple SPS (mSPS) and GSPS for the metrology data set with n=9635 data points.}
    \label{tbl:MSE_realdata_1}
\end{table}

\vspace{-0.5cm}
{According to the results reported in Table~\ref{tbl:MSE_realdata_1}, the best predictions are obtained when {the} number of observations per block is 500. We compare the GSPS method with 500 data points per block against the two approximation methods developed in \cite{Alvarez2011}, namely {the} Full Independent Training Conditional (FITC) {method and the} Partially Independent Training Conditional (PITC) {method}. For different number of inducing points $K\in\{100, 500, 1000\}$, we ran both methods on the data set. The locations of the inducing points along with the hyper-parameters of their model are found by maximizing the likelihood through a scaled conjugate gradient method as proposed by~\cite{Alvarez2011}. Initially, the inducing points are located completely at random.}

\begin{table}[h!]
    \begin{center}
    \begin{tabular}{lcc}
    \toprule
     Method & MSPE & std. error \\
     \hline
     mSPS (n/block=500) & 0.0621 & 0.0021 \\
     \hline
     GSPS (n/block=500) & 0.0167 & 0.0012 \\
     \hline
     FITC (K=100)     & 0.0551 & 0.0042 \\
     FITC (K=500)     & 0.0463 & 0.0011 \\
     FITC (K=1000)   & 0.0174 & 0.0010 \\
     \hline
     PITC (K=100)     & 0.0698 & 0.0062 \\
     PITC (K=500)     & 0.0421 & 0.0021 \\
     PITC (K=1000)   & 0.0197 & 0.0007 \\
    \bottomrule
    \end{tabular}
    \end{center}
    \vspace{-0.2cm}
     \caption{10-fold cross validation to compare prediction performance of mSPS, GSPS vs. FITC and PITC methods by Alvarez and Lawrence (2011) for the metrology data set with n=9635 data points}
    \label{tbl:MSE_realdata_2}
    \vspace{-0.5cm}
\end{table}

{{Intuitively, the} best prediction performance for both FITC and PITC approximations are obtained for {the larger $K$ values as this represents a better approximation of the underlying GRF}. {The GSPS method is performing better than FITC and PITC for all $K$ parameter choice.} 
Finally, as expected, {fitting $p$ univariate GRF models (mSPS)} is performing {worse than} the multivariate methods.}

\section{Conclusions and future research}
\noindent 
A new two-stage estimation method is proposed to fit multivariate Gaussian Random Field (GRF) models with separable covariance functions. Theoretical convergence {rates} for the estimated between-response covariance matrix and the estimated correlation function parameter are established {with respect to the number of process realizations}. Numerical studies confirm the theoretical results. From {a} statistical perspective, the first stage provides a Gaussian Markov Random Field (GMRF) approximation to the underlying GRF without discretizing the input space or assuming a sparsity structure for the precision matrix. From an optimization perspective, the first stage helps to ``zoom into" the region where the global optimal covariance parameters exist, facilitating the second stage least-squares optimization.

In this research, we considered separable covariance functions. Future research may consider non-separable covariance functions, e.g., {convolutions} of covariance functions, or kernel {convolutions}. {As another potential future work, we also propose estimating {the cross-covariance matrix} $\hat{\Gamma}$ at the {outset} by solving $\hat{\Gamma}=\argmin_\Gamma\{\norm{\Gamma-\frac{1}{n}\sum_{i=1}^n S^{ii}}_F:\ \Gamma\succeq\epsilon \mathbf{I}\}$. Then we propose solving the following problem as the new STAGE-I: \vspace*{-4mm}
{\small
\begin{equation*}
\begin{split}
\hat{P}_\rho = &\argmin_{P_\rho}
\fprod{S,P_\rho\otimes\hat{\Gamma}^{-1}}-\log\det(P_\rho\otimes \hat{\Gamma}^{-1})+\alpha\fprod{G\otimes(\ones_p\ones_p^\top),|P_\rho\otimes \hat{\Gamma}^{-1}|} \\
       & \quad \hbox{s.t.} \qquad {a^*}{\lambda_{\max}(\hat{\Gamma})} I \preceq P_\rho\preceq {b^*}{\lambda_{\min}(\hat{\Gamma})} I.
\end{split}
\vspace*{-4mm}
\end{equation*}
}
\noindent Note that $\log\det(P_\rho\otimes \hat{\Gamma}^{-1})=p\log\det(P_\rho)-n\log\det(\hat{\Gamma})$. Hence, there exists some $S_\rho,~G_\rho\in\mathbb{S}^n$, which can be computed very efficiently, such that
{\small
\begin{equation*}
\hat{P}_\rho = \argmin_{P_\rho}\left\{\fprod{S_\rho,P_\rho}-p\log\det(P_\rho)+\alpha\fprod{G_\rho,|P_\rho|}:\ {a^*}{\lambda_{\max}(\hat{\Gamma})} I \preceq P_\rho\preceq {b^*}{\lambda_{\min}(\hat{\Gamma})}I\right\}.
\vspace*{-3mm}
\end{equation*}}%
{Such an approach would be} much easier to solve in terms of computational complexity -- the overall complexity is $\cO(\log(1/\epsilon)n^3)$ for this STAGE-I problem. {Further work could be devoted to proving} consistency of the resulting estimator {and} its rate {could be} compared with {the} $\log(1/\epsilon^2)$ of GSPS.}

\markboth{\hfill{\footnotesize\rm Sam Davanloo Tajbakhsh, Necdet Serhat Aybat, and Enrique Del Castillo} \hfill}
{\hfill {\footnotesize\rm Generalized SPS Algorithm} \hfill}
{\singlespacing
\bibhang=1.7pc
\bibsep=2pt
\fontsize{9}{14pt plus.8pt minus .6pt}\selectfont
\renewcommand\bibname{}

\vskip 2cm
\noindent
Dept. of Integrated Systems Engineering, Columbus OH 43210 USA

\noindent
E-mail: davanloo-tajbakhsh.1@osu.edu
\vskip 18pt

\noindent
Dept. of Industrial and Manufacturing Engineering, University Park PA 16802 USA

\noindent
E-mail: nsa10@psu.edu
\vskip 18pt

\noindent
Dept. of Industrial and Manufacturing Engineering, University Park PA 16802 USA

\noindent
E-mail: exd13@psu.edu
\vskip 18pt
}



\newpage

\section{Appendix}\mbox{}

\noindent {\bf Proof of Theorem~\ref{thm:statAnalysis1}}\\
The proof given below is a slight modification of the proof of Theorem~3.1 in~\cite{SPSarxiv} to obtain tighter bounds. For the sake of completeness, we provide the proof. Through the change of variables $\Delta:=P-P^*$, we can write \eqref{eq:convexProgram} in terms of $\Delta$ as \vspace*{-2mm}
\begin{small}
\begin{equation*}
\hat{\Delta} = \argmin \{F(\Delta):=\fprod{S,\Delta+P^*}-\log\det(\Delta+P^*)+\alpha\fprod{G\otimes(\ones_p\ones_p^\top),|\Delta+P^*|}: \Delta\in\cF\},
\vspace*{-2mm}
\end{equation*}
\end{small}
\vspace*{-8mm}

\noindent where $\cF:=\{\Delta\in\reals^{np\times np}:\ \Delta=\Delta^\top, \ a^*\mb{I}\preceq \Delta+P^* \preceq b^*\mb{I}\}$. Note that $\hat{\Delta}=\hat{P}-P^*$. Define $g(\Delta):=-\log\det(\Delta+P^*)$ on $\cF$. $g(.)$ is strongly convex over $\cF$ with modulus $1/{b^*}^2$; hence, for any $\Delta\in\cF$, it follows that
$g(\Delta)-g(\mathbf{0}) \geq -\fprod{{P^*}^{-1},\Delta}+\frac{1}{2{b^*}^2}\norm{\Delta}_F^2$.
Let $H(\Delta):=F(\Delta)-F(\mathbf{0})$ and $S_{\Delta}:=\{\Delta\in\cF: \norm{\Delta}_F > 2{b^*}^2p(n+\norm{G}_F)\alpha\}$. Under probability event $\Omega=\{\norm{\vect(S^{ij}-\Sigma^{ij})}_\infty\leq\alpha,\ \forall(i,j)\in\cI\times\cI\}$, for any $\Delta\in S_{\Delta}\subset\cF$, 
\vspace*{-2mm}
\begin{small}
\begin{align*}
H(\Delta) &\geq \fprod{S,\Delta}-\fprod{{P^*}^{-1},\Delta}+\frac{1}{2{b^*}^2}\norm{\Delta}_F^2+\alpha\fprod{G\otimes(\ones_p\ones_p^\top),|\Delta+P^*|}-\alpha\fprod{G,|P^*|}\\
          &\geq \frac{1}{2{b^*}^2}\norm{\Delta}_F^2+\fprod{\Delta,S-C^*}-\alpha\fprod{G\otimes(\ones_p\ones_p^\top),|\Delta|}\\
          &\geq \frac{1}{2{b^*}^2}\norm{\Delta}_F^2-\alpha p(n+\norm{G}_F)\norm{\Delta}_F> 0,
          \vspace*{-4mm}
\end{align*}
\end{small}
\vspace*{-10mm}

\noindent where the second inequality follows from the triangle inequality, the third one holds under the probability event $\Omega$ and follows from the Cauchy-Schwarz inequality, and the final strict one follows from the definition of $S_{\Delta}$. Since $F(\mathbf{0})$ is a constant, $\hat{\Delta} = \argmin \{H(\Delta): \Delta\in\cF\}$. Hence, $H(\hat{\Delta})\leq H(\mathbf{0})=0$. Therefore, 
$\hat{\Delta}\not\in S_{\Delta}$ under the probability event $\Omega$. It is important to note that $\hat{\Delta}$ satisfies the first two conditions given in the definition of $S_{\Delta}$. This implies $\norm{\hat{\Delta}}_F\leq 2{b^*}^2p(n+\norm{G}_F)\alpha$ whenever the probability event $\Omega$ is true. Hence, \vspace*{-3mm}
\begin{small}
\begin{align*}
\mbox{Pr}\left(\norm{\hat{P}-P^*}_F\leq 2 {b^*}^2 p (n+\norm{G}_F) \alpha\right) &\geq \mbox{Pr}\left(\norm{\vect(S^{ij}-\Sigma^{ij})}_\infty\leq\alpha,\  \forall(i,j)\in\cI\times\cI\right) \\
            &= 1-\mbox{Pr}\left(\max_{i,j\in\cI}\norm{\vect(S^{ij}-\Sigma^{ij})}_\infty>\alpha\right)\\
            &\geq 1-\sum_{i,j\in\cI}\mbox{Pr}\left(\norm{\vect(S^{ij}-\Sigma^{ij})}_\infty>\alpha\right).
\vspace*{-3mm}
\end{align*}
\end{small}
\vspace*{-12mm}

\noindent Recall that $S=\frac{1}{N}\sum_{r=1}^N\by^{(r)}{\by^{(r)}}^\top$ and $\by^{(r)}=[y_i^{(r)}]_{i\in\cI}$ for $r=1,\ldots,N$. Note $\Sigma^{ii}=\Gamma^*$ for $i\in\cI$; hence, $y_i^{(r)}\sim\cN(\mathbf{0},\Gamma^*)$, i.e., multivariate Gaussian with mean $\mathbf{0}$ and covariance matrix $\Gamma^*$, for all $i$ and $r$. 
Therefore, Lemma~1 in~\cite{Ravikumar11high} implies
$\mbox{Pr}\left(\norm{\vect(S^{ij}-\Sigma^{ij})}_\infty>\alpha \right)\leq B_\alpha$
for $\alpha\in\left(0,40\max_i\Gamma^*_{ii}\right)$, where 
\begin{equation*}
B_\alpha:=4p^2\exp\left(\frac{-N}{2}\left(\frac{\alpha}{40\max_i \Gamma^*_{ii}}\right)^2\right). 
\end{equation*}
Hence, given any $M>0$, by requiring 
\newline $N \geq \left(\frac{40\max_i\Gamma^*_{ii}}{\alpha}\right)^2 N_0$, we get $B_\alpha\leq \frac{1}{n^2}(np)^{-M}$. 
Thus, for any $N\geq N_0$, we have 
\begin{equation*}
\sum_{i,j\in\cI}\mbox{Pr}\left(\norm{\vect(S^{ij}-\Sigma^{ij})}_\infty>\alpha\right)\leq (np)^{-M}
\end{equation*}
for all $40\max_i\Gamma^*_{ii}\sqrt{\frac{N_0}{N}}\leq\alpha\leq 40\max_i\Gamma^*_{ii}$.
\par
\noindent {\bf Proof of Theorem~\ref{thm:main}}\\ \vspace*{-3mm}
For the sake of simplicity of the notation let $\Phi=(\Gamma,C)\in\mS^n\times\mS^{np}$, and define $\norm{(\Gamma,C)}_a:=\max\{\norm{\Gamma}_2,\norm{C}_2\}$ over the product vector space $\mS^n\times\mS^{np}$; also let $\Psi=(\th,\Gamma,C)\in\reals^q\times\mS^n\times\mS^{np}$, and define $\norm{(\th,\Gamma,C)}_b:=\norm{\th}+\norm{(\Gamma,C)}_a$ over the product vector space $\reals^q\times\mS^n\times\mS^{np}$. Throughout the proof $\hat{\Phi}:=(\hat{\Gamma},\hat{C})$, $\Phi^*:=(\Gamma^*,C^*)$, and $\hat{\Psi}:=(\hat{\th},\hat{\Phi})$, $\Psi^*:=(\th^*,\Phi^*)$.

As $\th^*\in\intr(\Theta)$, there exists $\delta_1>0$ such that $\B{2}{\th^*,\delta_1}\subset\Theta$. Moreover, since $\rho(\mb{x},\mb{x}';\th)$ is twice continuously differentiable in $\th$ over $\Theta$ for all $\mb{x},\mb{x}'\in\cX$, $R:\Theta\rightarrow\mS^n$ is also twice continuously differentiable. Hence, from \eqref{eq:hessian}, it follows that $\grad^2 f(\th;\Gamma,C)$ is continuous in $\Psi=(\th,\Gamma,C)$; and since eigenvalues of a matrix are continuous functions of matrix entries, $\lambda_{\min}\left(\grad^2 f(\th;\Gamma,C)\right)$ is continuous in $\Psi$ on $\B{b}{\Psi^*,\delta_1}$ as well. Therefore, it follows from Lemma~\ref{lem:strong_convexity} that there exists $0<\delta_2\leq \delta_1$ such that $\grad_\th^2 f(\th;\Gamma,C)\succeq \frac{\gamma^*}{2}I$ for all $\Psi=(\th,\Gamma,C)\in\B{b}{\Psi^*,\delta_2}$.

Let $\cQ:=\Bc{a}{\Phi^*,\tfrac{1}{2}\delta_2}$ and $\Theta':=\Theta\cap\Bc{2}{\th^*,\tfrac{1}{2}\delta_2}$, i.e., \vspace*{-5mm}
\begin{align}
\label{eq:sets_def}
\cQ&=\{(\Gamma, C):\ \max\{\norm{\Gamma-\Gamma^*}_2,~\norm{C-C^*}_2\}\leq\tfrac{1}{2}\delta_2\},\\
\Theta'&=\{\th\in\Theta:\ \norm{\th-\th^*}\leq\tfrac{1}{2}\delta_2\}.
\end{align}
\vspace*{-15mm}

\noindent Clearly $f$ is strongly convex in $\th$ over $\Theta'$ with convexity modulus $\frac{\gamma^*}{2}$ for all $(\Gamma,C)\in\cQ$. Define the unique minimizer over $\Theta'$: \vspace*{-5mm}
\begin{equation}
\label{eq:parametric_min}
\th(\Gamma,C):=\argmin_{\th\in\Theta'}f(\th;\Gamma,C).
\vspace*{-3mm}
\end{equation}
Since $\Theta'$ is a convex compact set and $f(\th;\Gamma,C)$ is jointly continuous in $\Psi=(\th,\Gamma,C)$ on $\Theta'\times\cQ$, from Berge's Maximum Theorem -- see~\cite{OK07real}, $\th(\Gamma,C)$ is continuous at $(\Gamma^*,C^*)$ and $\th(\Gamma^*,C^*)=\th^*$. Therefore, for any $0<\epsilon\leq\tfrac{1}{2}\delta_2$, there exists $\delta(\epsilon)>0$ such that $\delta(\epsilon)\leq\tfrac{1}{2}\delta_2$ and $\norm{\th(\Gamma,C)-\th^*}<\epsilon$ for all $\Phi=(\Gamma,C)$ satisfying $\norm{\Phi-\Phi^*}_a<\delta(\epsilon)$.

Fix some arbitrary $\epsilon\in(0,\tfrac{1}{2}\delta_2]$. Let $\hat{P}(\epsilon)$ be computed as in \eqref{eq:multVarCovexProgram} with \newline 
$\alpha(\epsilon)=40\max\limits_{i=1,...,p}(\Gamma_{ii}^*)\sqrt{\frac{N_0}{N(\epsilon)}}$ where sample size $N(\epsilon)$ denotes the number of process realizations (chosen depending on $\epsilon>0$). Hence, Theorem~\ref{thm:gammaConvergence} implies that by choosing $N(\epsilon)$ sufficiently large, we can guarantee that $\hat{C}(\epsilon)=\hat{P}(\epsilon)^{-1}$, and $\hat{\Gamma}(\epsilon)$ defined as in \eqref{eq:gamma_hat} satisfy \vspace*{-5mm}
\begin{equation}
\label{eq:ball_ineq_1}
\max\{\norm{\hat{C}(\epsilon)-C^*}_2,~\norm{\hat{\Gamma}(\epsilon)-\Gamma^*}_2\}<\delta(\epsilon)\leq\tfrac{1}{2}\delta_2,
\vspace*{-3mm}
\end{equation}
i.e., $\norm{\hat{\Phi}-\Phi^*}_a<\delta(\epsilon)$, with high probability. In the rest of the proof, for the sake of notational simplicity, we do not explicitly show the dependence on the fixed tolerance $\epsilon$; instead we simply write $\hat{P}$, $\hat{C}$, and $\hat{\Gamma}$.

Note that due to the parametric continuity discussed above, \eqref{eq:ball_ineq_1} implies that $\norm{\th(\hat{\Gamma},\hat{C})-\th^*}<\epsilon\leq \tfrac{1}{2}\delta_2$. Hence, the norm-ball constraint in the definition of $\Theta'$ will not be tight when $f(\th;\hat{\Gamma},\hat{C})$ is minimized over $\th\in\Theta'$, i.e., 
\begin{equation*}
\th(\hat{\Gamma},\hat{C})=\argmin_{\th\in\Theta'}f(\th;\hat{\Gamma},\hat{C})=\argmin_{\th\in\Theta}f(\th;\hat{\Gamma},\hat{C})=:\hat{\th},
\end{equation*}
see \eqref{eq:corrParamEst} for the definition of $\hat{\th}$. Therefore, $\norm{\hat{\Psi}-\Psi^*}_b<\delta_2\leq\delta_1$, i.e., \vspace*{-5mm}
\begin{equation}
\label{eq:ball_ineq_2}
\norm{\hat{\th}-\th^*}+\norm{(\hat{\Gamma},\hat{C})-(\Gamma^*,C^*)}_a<\delta_2\leq\delta_1.
\vspace*{-3mm}
\end{equation}
This implies that $\hat{\th}\in\intr\Theta$; thus, $\grad_\th f(\hat{\th};\hat{\Gamma},\hat{C})=\mathbf{0}$.

Although one can establish a direct relation between $\delta(\epsilon)$ and $\epsilon$ by showing that $\th(\Gamma,C)$ is Lipschitz continuous around $\th^*$, we will show a more specific result by upper bounding the error $\norm{\hat{\th}-\th^*}$ using $\norm{\hat{\Phi}-\Phi^*}_a$. Indeed, since $(\hat{\Gamma},\hat{C})\in\cQ$, $f(\th;\hat{\Gamma},\hat{C})$ is strongly convex in $\th\in\Theta'$ with modulus $\tfrac{1}{2}\gamma^*$; hence, $\th^*\in\Theta'$ and $\hat{\th}\in\Theta'$ imply that \vspace*{-5mm}
\begin{align}
\frac{\gamma^*}{2}\norm{\hat{\th}-\th^*}^2
&\leq \fprod{\grad_\th f(\th^*;\hat{\Gamma},\hat{C})-\grad_\th f(\hat{\th};\hat{\Gamma},\hat{C}),~\th^*-\hat{\th}} \nonumber\\
& = \fprod{\grad_\th f(\th^*;\hat{\Gamma},\hat{C})-\grad_\th f(\th^*;\Gamma^*,C^*),~\th^*-\hat{\th}},\label{eq:key_ineq_gradg_vs_theta}
\end{align}
\vspace*{-15mm}

\noindent where the equality follows from the fact that $\grad_\th f(\th^*;\Gamma^*,C^*)=\grad_\th f(\hat{\th};\hat{\Gamma},\hat{C})=\mathbf{0}$. Next, from \eqref{eq:grad_f_explicit} it follows that \vspace*{-5mm}
\begin{eqnarray*}
\lefteqn{\Delta_k:=\left|\frac{\partial}{\partial\theta_k}f(\th^*;\hat{\Gamma},\hat{C})-\frac{\partial}{\partial\theta_k} f(\th^*;\Gamma^*,C^*)\right|}\\
& &\leq \left|(\norm{\hat{\Gamma}}_F^2-\norm{\Gamma^*}_F^2)\fprod{R'_k(\th^*),R(\th^*)}+\fprod{C^*,R'_k(\th^*)\otimes\Gamma^*}
-\langle\hat{C},R'_k(\th^*)\otimes\hat{\Gamma}\rangle\right|\\
& &\leq\left(\norm{\hat{\Gamma}+\Gamma^*}_*\norm{R(\th^*)}_*+\norm{\hat{C}}_*\right)\norm{R'_k(\th^*)}_2\norm{\hat{\Gamma}-\Gamma^*}_2+n\norm{\Gamma^*}_*\norm{R'_k(\th^*)}_2\norm{\hat{C}-C^*}_2,
\end{eqnarray*}
\vspace*{-15mm}

\noindent where the second inequality uses the following basic inequalities and identities: Given $X,Y,V,W\in\reals^{m\times n}$ \textbf{i}) $\fprod{X,Y}\leq\norm{X}_2\norm{Y}_*$, \textbf{ii}) $\norm{X}_F^2-\norm{Y}_F^2=\fprod{X+Y,X-Y}$, \textbf{iii}) $\fprod{X,Y}-\fprod{V,W}=\fprod{X,Y-W}+\fprod{W,X-V}$; given $X\in\mS^p$, $Y\in\mS^n$ \textbf{iv}) $\norm{X\otimes Y}_2=\norm{X}_2\norm{Y}_2$, \textbf{v}) $\norm{X\otimes Y}_*\leq\min\{p\norm{X}_2\norm{Y}_*,~n\norm{X}_*\norm{Y}_2\}$. Note that since $R(\th^*)\in\mS^n_{++}$, $\norm{R(\th^*)}_*=\Tr(R(\th^*))=n$. Moreover, \eqref{eq:ball_ineq_1} implies that $\norm{\hat{\Gamma}}_*\leq \norm{\Gamma^*}_*+\tfrac{p}{2}\delta_2$, and $\norm{\hat{C}}_*\leq \norm{C^*}_*+\tfrac{np}{2}\delta_2$. Hence, \vspace*{-5mm}
\begin{align*}
\Delta_k\leq \left(3n\norm{\Gamma^*}_*+\norm{C^*}_*+\frac{(np+1)}{2}\delta_2\right)\norm{R_k'(\th^*)}_2\norm{(\hat{\Gamma},\hat{C})-(\Gamma^*,C^*)}_a.
\end{align*}
\vspace*{-15mm}

\noindent Therefore, for $\kappa:=\left(3n\norm{\Gamma^*}_*+\norm{C^*}_*+\frac{(np+1)}{2}\delta_2\right)\left(\sum_{k=1}^q\norm{R_k'(\th^*)}_2^2\right)^{\tfrac{1}{2}}$
\vspace*{-5mm}
\begin{equation*}
\norm{\grad_\th f(\th^*;\hat{\Gamma},\hat{C})-\grad_\th f(\th^*;\Gamma^*,C^*)}_2\leq\kappa~\norm{(\hat{\Gamma},\hat{C})-(\Gamma^*,C^*)}_a
\vspace*{-3mm}
\end{equation*}
Applying Cauchy Schwarz inequality to \eqref{eq:key_ineq_gradg_vs_theta}, we have \vspace*{-5mm}
\begin{equation}
\norm{\hat{\th}-\th^*}\leq 2\frac{\kappa}{\gamma^*}~\norm{(\hat{\Gamma},\hat{C})-(\Gamma^*,C^*)}_a.
\vspace*{-3mm}
\end{equation}
Thus, choosing $N(\epsilon)\geq N_0:=\left\lceil2\big[(M+2)\ln(np)+\ln 4\big]\right\rceil$ such that \vspace*{-5mm}
\begin{equation*}
\sqrt{\frac{N(\epsilon)}{N_0}}\geq 160\max\limits_{i=1,...,p}(\Gamma_{ii}^*)\frac{\kappa}{\gamma^*} \left(\frac{b^*}{a^*}\right)^2p(n+\norm{G}_F)~\frac{1}{\epsilon},
\vspace*{-3mm}
\end{equation*}
i.e., $N(\epsilon)=\cO(\frac{1}{\epsilon^2})$, implies that $\norm{\hat{\th}-\th^*}\leq\epsilon$, and $\norm{\hat{\Gamma}-\Gamma^*}_2\leq\frac{\gamma^*}{2\kappa}\epsilon$ with probability at least $1-(np)^{-M}$.
\par

\end{document}

%% file: defs.tex
\def\grad{\nabla}

\def\br{\mathbf{r}}

\def\bx{\mathbf{x}}  
\def\by{\mathbf{y}}

\def\bo{\mathbf{0}}

\def\th{{\boldsymbol{\theta}}}
\def\Th{{\boldsymbol{\Theta}}}

\def\cB{\mathcal{B}}

\def\cD{\mathcal{D}}
\def\cE{\mathcal{E}}
\def\cF{\mathcal{F}}

\def\cI{\mathcal{I}}

\def\cN{\mathcal{N}}
\def\cO{\mathcal{O}}

\def\cQ{\mathcal{Q}}

\def\cV{\mathcal{V}}

\def\cX{\mathcal{X}}

\def\mR{\mathbb{R}}
\def\mS{\mathbb{S}}

\def\smskip{\smallskip}

\def\texitem#1{\par\smskip\noindent\hangindent 25pt
               \hbox to 25pt {\hss #1 ~}\ignorespaces}

\def\norm#1{\|#1\|}

\newcommand{\BEAS}{\begin{eqnarray*}}
\newcommand{\EEAS}{\end{eqnarray*}}
\newcommand{\BEA}{\begin{eqnarray}}
\newcommand{\EEA}{\end{eqnarray}}
\newcommand{\BEQ}{\begin{eqnarray}}
\newcommand{\EEQ}{\end{eqnarray}}
\newcommand{\BIT}{\begin{itemize}}
\newcommand{\EIT}{\end{itemize}}
\newcommand{\BNUM}{\begin{enumerate}}
\newcommand{\ENUM}{\end{enumerate}}

\newcommand{\BA}{\begin{array}}
\newcommand{\EA}{\end{array}}

\newcommand{\ones}{\mathbf 1}

\newcommand{\reals}{\mathbb{R}}
\newcommand{\integers}{\mathbb{Z}}

\newcommand{\Rank}{\mathop{\bf rank}}
\newcommand{\Tr}{\mathop{\bf Tr}}

\newcommand{\intr}{\mathop{\bf int}}

\newcommand{\vect}{\mathop{\bf vec}}

\newif\ifpagenumbering
\pagenumberingtrue

\pagenumberingfalse

\newsavebox{\theorembox}
\newsavebox{\lemmabox}
\newsavebox{\assbox}
\savebox{\theorembox}{\noindent\bf Theorem}
\savebox{\lemmabox}{\noindent\bf Lemma}
\savebox{\assbox}{\noindent\bf Assumption}